%% file: main.tex
\documentclass[runningheads]{llncs}
\usepackage[T1]{fontenc}
\usepackage{graphicx}
\usepackage{algorithm,algpseudocode}
\usepackage{booktabs}
\usepackage{amsmath}
\usepackage{amssymb}
\usepackage{thmtools}
\usepackage{thm-restate}
\usepackage[colorlinks]{hyperref}
\hypersetup{
	colorlinks=true,
	linkcolor=blue,
	filecolor=magenta,
	urlcolor=blue,
	citecolor = blue,
}
\usepackage{cleveref}

\usepackage[colorinlistoftodos,prependcaption,textsize=tiny]{todonotes}
\input{macros}

\newcommand{\ApproxMCSix}{\ensuremath{\mathsf{RoundMC}}}
\newcommand{\ApproxMCSixCore}{\ensuremath{\mathsf{RoundMCCore}}}
\newcommand{\ganak}{\ensuremath{\mathsf{Ganak}}}
\newcommand{\Arjun}{\ensuremath{\mathsf{Arjun}}}
\newcommand{\hashspace}{\ensuremath{\mathcal{H}}}

\newcommand{\roundvalue}{\ensuremath{\mathsf{roundValue}}}
\newcommand{\roundtype}{\ensuremath{\mathsf{roundUp}}}
\newcommand{\configround}{\ensuremath{\mathsf{configRound}}}
\newcommand{\computeIter}{\ensuremath{\mathsf{computeIter}}}
\newcommand{\finalestimate}{\ensuremath{\mathsf{finalEstimate}}}

\begin{document}
\title{Rounding Meets Approximate Model Counting}
\author{Jiong Yang
	\and
	Kuldeep S. Meel
}
\authorrunning{J. Yang and K. S. Meel}
\institute{National University of Singapore}
\maketitle              %
\input{sec/abstract}
\input{sec/introduction}
\input{sec/preliminary}

\input{sec/related_work}

\input{sec/weakness}

\input{sec/theory}

\input{sec/experiment}

\input{sec/conclusion}

\bibliographystyle{splncs04}
\bibliography{main}

\input{sec/appendix}
\end{document}

%% file: macros.tex
\algblockdefx{MRepeat}{EndRepeat}{\textbf{repeat}}{}
\algnotext{EndRepeat}
\algnotext{EndIf}
\algnewcommand{\algorithmicgoto}{\textbf{go to line}}%
\algnewcommand{\Goto}[1]{\algorithmicgoto~\ref{#1}}%

\algnotext{EndFor}
\algnotext{EndIf}
\algnotext{EndWhile}
\algnotext{EndFunction}
\algblockdefx{Do}{EndDo}{\textbf{do}\;}{}
\algnotext{EndDo}

\newcommand{\Cell}[2]{\ensuremath{\mathsf{Cell}_{\langle #1, #2 \rangle}}}
\newcommand{\satisfying}[1]{\ensuremath{\mathsf{sol({#1})}}} %

\newcommand{\BoundedSAT}{\ensuremath{\mathsf{BoundedSAT}}}

\newcommand{\PAC}{\ensuremath{\mathsf{PAC}}}

\newcommand{\NP}{\ensuremath{\mathsf{NP}}}

\newcommand{\ApproxCount}{\ensuremath{\mathsf{ApproxCount}}}

\newcommand{\sharpSAT}{\ensuremath{\#\mathsf{SAT}}}
\newcommand{\sharpP}{\ensuremath{\#\mathsf{P}}}

\newcommand{\ApproxMCFour}{\ensuremath{\mathsf{ApproxMC4}}}

\newcommand{\solCount}{\ensuremath{\mathsf{nSols}}}

\newcommand{\emptyList}{\ensuremath{\mathsf{emptyList}}}

\newcommand{\ApproxMC}{\ensuremath{\mathsf{ApproxMC}}}
\newcommand{\ApproxMCCore}{\ensuremath{\mathsf{ApproxMCCore}}}

\newcommand{\FibBinSearch}{\ensuremath{\mathsf{LogSATSearch}}}

\newcommand{\iter}{\ensuremath{\mathsf{iter}}}

\newcommand{\AddToList}{\ensuremath{\mathsf{AddToList}}}
\newcommand{\FindMedian}{\ensuremath{\mathsf{FindMedian}}}

\newcommand{\Vars}{\ensuremath{\mathsf{Vars }}}

\newcommand{\thresh}{\ensuremath{\mathsf{thresh}}}

\newcommand{\pivot}{\ensuremath{\mathsf{pivot}}}

\newcommand{\CNF}{\ensuremath{\mathsf{CNF}}}

\newcommand{\PH}{\mbox{{\sf PH}}}

\newcommand{\Prob}[1]{\ensuremath{\mathsf{Pr}\left[{#1}\right]}}

\newcommand{\bigO}[1]{\ensuremath{\mathcal{O}\left({#1}\right)}}

\newcommand{\Exp}[1]{\ensuremath{\mathsf{E}\left[{#1}\right]}}

\newcommand{\Obad}{\mathsf{Error}}

\newcommand{\numVars}{\ensuremath{{n}}}

\newcommand{\Prb}[1]{\text{Pr}\left[#1\right]}

\newcommand{\SatisfyingHashSet}[3]{\ensuremath{\mathsf{Cell}_{\langle #1, #2, #3 \rangle}}}

\newcommand{\Cnt}[2]{\ensuremath{\mathsf{Cnt}_{\langle #1, #2 \rangle}}}

\newcommand{\PP}{\ensuremath \mathcal{P}}

%% file: sec/abstract.tex
\begin{abstract}
The problem of model counting, also known as {\sharpSAT}, is to compute the number of models or satisfying assignments of a given Boolean formula $F$. 
Model counting is a fundamental problem in computer science with a wide range of applications. 
In recent years, there has been a growing interest in using hashing-based techniques for approximate model counting that provide $(\varepsilon, \delta)$-guarantees: i.e., the count returned is within a $(1+\varepsilon)$-factor of the exact count with confidence at least $1-\delta$. While hashing-based techniques attain reasonable scalability for large enough values of $\delta$, their scalability is severely impacted for smaller values of $\delta$, thereby preventing their adoption in application domains that require estimates with high confidence. 
 
The primary contribution of this paper is to address the Achilles heel of hashing-based techniques: we propose a novel approach based on {\em rounding} that allows us to achieve a significant reduction in runtime for smaller values of $\delta$. The resulting counter, called {\ApproxMCSix}, achieves a substantial runtime performance improvement over the current state-of-the-art counter, {\ApproxMC}. In particular, our extensive evaluation over a benchmark suite consisting of 1890 instances shows that {\ApproxMCSix} solves 204 more instances than {\ApproxMC}, and achieves a $4\times$ speedup over {\ApproxMC}.

\end{abstract}

%% file: sec/introduction.tex
\section{Introduction} \label{sec: introduction}
Given a Boolean formula $F$, the problem of model counting is to compute the number of models of $F$.
Model counting is a fundamental problem in computer science with a wide range of applications, such as control improvisation~\cite{GVF22}, network reliability~\cite{V79,DMPV17}, 
neural network verification~\cite{BSSM+19}, probabilistic reasoning~\cite{R96,SBK05,CFMS+14,EGSS13a}, and the like. In addition to myriad applications, the problem of model counting is a fundamental problem in theoretical computer science. In his seminal paper, Valiant showed that $\sharpSAT$ is $\sharpP$-complete, where $\sharpP$ is the set of counting problems whose decision versions lie in {\NP}~\cite{V79}. Subsequently, Toda demonstrated the theoretical hardness of the problem by showing that every problem in the entire polynomial hierarchy can be solved by just one call to a $\sharpP$ oracle; more formally, $\PH \subseteq \mbox{{\sf P}}^{\sharpP}$~\cite{T89}.

Given the computational intractability of $\sharpSAT$, there has been sustained interest in the development of approximate techniques from theoreticians and practitioners alike. Stockmeyer introduced a randomized hashing-based technique that provides $(\varepsilon, \delta)$-guarantees (formally defined in Section~\ref{sec: preliminary}) given access to an {\NP} oracle~\cite{S83}. Given the lack of practical solvers that could handle problems in {\NP} satisfactorily, there were no practical implementations of Stockmeyere's hashing-based techniques until the 2000s~\cite{GSS06}. Building on the unprecedented advancements in the development of SAT solvers, Chakraborty, Meel, and Vardi extended Stockmeyer's framework to a scalable $(\varepsilon, \delta)$-counting algorithm, {\ApproxMC}~\cite{CMV13}. The subsequent years have witnessed a sustained interest in further optimizations of the hashing-based techniques for approximate counting~\cite{EGSS13b,EGSS13a,CFMS+14,IMMV16,MVCF+15,CMMV16,SM19,AM20,YM21,YCM22}. The current state-of-the-art technique for approximate counting is a hashing-based framework called {\ApproxMC}, which is in its fourth version, called {\ApproxMCFour}~\cite{SGM20,SM22}.

The core theoretical idea behind the hashing-based framework is to use 2-universal hash functions to partition the solution space, denoted by $\satisfying{F}$ for a formula $F$, into \emph{roughly equal small} cells, 
wherein a cell is considered \emph{small} if it contains solutions less than or equal to a pre-computed threshold, {\thresh}. 
An {\NP} oracle (in practice, an SAT solver) is employed to check if a cell is small by enumerating solutions one-by-one until either there are no more solutions or we have already enumerated $\thresh + 1$ solutions. 
Then, we randomly pick a cell, enumerate solutions within the cell (if the cell is small), and scale the obtained count by the number of cells to obtain an estimate for $|\satisfying{F}|$. 
To amplify the confidence, we rely on the standard {\em median technique}: repeat the above process, called {\ApproxMCCore}, multiple times and return the median. Computing the median amplifies the confidence as for the median of $t$ repetitions to be outside the desired range (i.e., $\left[ \frac{|\satisfying{F}|}{1+\varepsilon}, (1+\varepsilon) |\satisfying{F}| \right]$), it should be the case that at least half of the repetitions of {\ApproxMCCore} returned a wrong estimate. 

In practice, every subsequent repetition of {\ApproxMCCore} takes a similar time, and the overall runtime increases linearly with the number of invocations. The number of repetitions depends logarithmically on $\delta^{-1}$. As a particular example, for $\epsilon=0.8$, the number of repetitions of {\ApproxMCCore} to attain $\delta=0.1$ is 21, which increases to 117 for $\delta=0.001$: a significant increase in the number of repetitions (and accordingly, the time taken). Accordingly, it is no surprise that empirical analysis of tools such as {\ApproxMC} has been presented with a high delta (such as $\delta=0.1$). On the other hand, for several applications, such as network reliability, and quantitative verification, the end users desire estimates with high confidence. Therefore, the design of efficient counting techniques for small $\delta$ is a major challenge that one needs to address to enable the adoption of approximate counting techniques in practice. 

The primary contribution of our work is to address the above challenge. We introduce a new technique called {\em rounding} that enables dramatic reductions in the number of repetitions required to attain a desired value of confidence. The core technical idea behind the design of the {\em rounding} technique is based on the following observation: 
Let $L$ (resp. $U$) refer to the event that a given invocation of {\ApproxMCCore} under (resp. over)-estimates $|\satisfying{F}|$. For a median estimate to be wrong, either the event $L$ happens in half of the invocations of {\ApproxMCCore} or the event $U$ happens in half of the invocations of {\ApproxMCCore}. The number of repetitions depends on $\max(\Pr[L], \Pr[U])$. The current algorithmic design (and ensuing analysis) of {\ApproxMCCore} provides a weak upper bound on $\max\{\Pr[L], \Pr[U]\}$: in particular, the bounds on $\max\{\Pr[L], \Pr[U]\}$ and $\Pr[L \cup U]$ are almost identical.   Our key technical contribution is to design a new procedure, {\ApproxMCSixCore}, based on the rounding technique that allows us to obtain significantly better bounds on $\max\{\Pr[L], \Pr[U]\}$. 
     
The resulting algorithm, called {\ApproxMCSix}, follows a similar structure to that of {\ApproxMC}: it repeatedly invokes the underlying core procedure {\ApproxMCSixCore} and returns the median of the estimates.  
Since a single invocation of {\ApproxMCSixCore} takes as much time as {\ApproxMCCore}, the reduction in the number of repetitions is primarily responsible for the ensuing speedup. As an example, for $\varepsilon=0.8$, the number of repetitions of {\ApproxMCSixCore} to attain $\delta=0.1$ and $\delta=0.001$ is just 5 and 19, respectively; the corresponding numbers for {\ApproxMC} were 21 and 117.    
An extensive experimental evaluation on 1890 benchmarks shows that the rounding technique provided $4\times$ speedup than the state-of-the-art approximate model counter, {\ApproxMCFour}. Furthermore, for a given timeout of 5000 seconds, {\ApproxMCSix} solves 204 more instances than {\ApproxMCSix} and achieves a reduction of 1063 seconds in the PAR-2 score.

The rest of the paper is organized as follows. We introduce notation and preliminaries in Section~\ref{sec: preliminary}. 
To place our contribution in context, we review related works in Section~\ref{sec: related work}.
We identify the weakness of the current technique in Section~\ref{sec: weakness} and present the rounding technique in Section~\ref{sec: theory} to address this issue.
Then, we present our experimental evaluation in Section~\ref{sec: experiment}.
Finally, we conclude in Section~\ref{sec: conclusion}.

%% file: sec/preliminary.tex
\section{Notation and Preliminaries} \label{sec: preliminary}

Let $F$ be a Boolean formula in conjunctive normal form ($\CNF$), and let $\Vars(F)$ be the set of variables appearing in $F$. 
The set $\Vars(F)$ is also called the \emph{support} of $F$. 
An assignment $\sigma$ of truth values to the variables in $\Vars(F)$ is called a \emph{satisfying assignment} or \emph{witness} of $F$ if it makes $F$ evaluate to true. 
We denote the set of all witnesses of $F$ by $\satisfying{F}$. 
Throughout the paper, we will use $\numVars$ to denote $|\Vars(F)|$.

The \emph{propositional model counting problem} is to compute $|\satisfying{F}|$ for a given $\CNF$ formula $F$.
A \emph{probably approximately correct} (or $\PAC$) counter is a probabilistic algorithm $\ApproxCount(\cdot, \cdot, \cdot)$
that takes as inputs a formula $F$, a tolerance parameter $\varepsilon > 0$, and a confidence parameter $\delta \in (0,1]$, and returns an $(\varepsilon, \delta)$-estimate  $c$, 
i.e., $\Prob{\frac{|\satisfying{F}|}{1+\varepsilon} \le c \le (1 + \varepsilon)|\satisfying{F}|} \ge 1 - \delta$. 
$\PAC$ guarantees are also sometimes referred to as $(\varepsilon, \delta)$-guarantees.

A closely related notion is projected model counting, where we are interested in computing the cardinality of $\satisfying{F}$ projected on a subset of variables $\PP \subseteq \Vars(F)$. 
While for clarity of exposition, we describe our algorithm in the context of model counting, 
the techniques developed in this paper are applicable to projected model counting as well. Our empirical evaluation indeed considers such benchmarks.

\subsection{Universal Hash Functions}

Let $n, m \in \mathbb{N}$ and $\hashspace(n, m) \overset{\triangle}{=} \{h : \{0, 1\}^n \rightarrow \{0, 1\}^m\}$ be a family of hash functions mapping $\{0, 1\}^n$ to $\{0, 1\}^m$. 
We use $h \overset{R}{\leftarrow}\hashspace(n, m)$ to denote the probability space obtained by choosing a function $h$ uniformly at random from $\hashspace(n, m)$. 
To measure the quality of a hash function we are interested in the set of elements of $\satisfying{F}$ mapped to $\alpha$ by $h$, 
denoted $\SatisfyingHashSet{F}{h}{\alpha}$ and its cardinality, i.e., $|\SatisfyingHashSet{F}{h}{\alpha}|$.
We write $\Pr[Z:\Omega]$ to denote the probability of outcome $Z$ when sampling from a probability space $\Omega$. 
For brevity, we omit $\Omega$ when it is clear from the context. 
The expected value of $Z$ is denoted $\Exp{Z}$ and its variance is denoted $\sigma^2[Z]$. 

\begin{definition}
	A family of hash functions $\hashspace(n, m)$ is strongly 2-universal if $\forall x, y \in \{0, 1\}^n$, 
	$\alpha \in \{0, 1\}^m$, $h \overset{R}{\leftarrow} \hashspace(n, m)$,
	\begin{align*}
		\Prob{h(x) = \alpha} = \frac{1}{2^m} = \Prob{h(x) = h(y)}
	\end{align*}
\end{definition}
For $h \overset{R}{\leftarrow} \hashspace(n, n)$ and $\forall m \in \{1, ..., n\}$, the $m^{th}$ prefix-slice of $h$, denoted $h^{(m)}$, is a map from $\{0, 1\}^n$ to $\{0, 1\}^m$, 
such that $h^{(m)}(y)[i] = h(y)[i]$, for all $y \in \{0, 1\}^n$ and for all $i \in \{1, ..., m\}$. 
Similarly, the $m^{th}$ prefix-slice of $\alpha \in \{0, 1\}^n$, denoted $\alpha^{(m)}$, is an element of $\{0, 1\}^m$ 
such that $\alpha^{(m)}[i] = \alpha[i]$ for all $i \in \{1, ..., m\}$. 
To avoid cumbersome terminology, we abuse notation and write $\Cell{F}{m}$(resp. $\Cnt{F}{m}$) as a short-hand for $\SatisfyingHashSet{F}{h^{(m)}}{\alpha^{(m)}}$
(resp. $|\SatisfyingHashSet{F}{h^{(m)}}{\alpha^{(m)}}|$).
The following proposition presents two results that are frequently used throughout this paper.
The proof is deferred to~\Cref{sec: appendix preliminary}.
\begin{proposition}
	\label{prop: 2 universal}
	For every $1 \le m \le n$, the following holds:
	\begin{align}
		\label{eq: exp}
		\Exp{\Cnt{F}{m}} = \frac{|\satisfying{F}|}{2^m}
	\end{align}
	\begin{align}
		\label{eq: var exp}
		\sigma^2\left[\Cnt{F}{m}\right] \le \Exp{\Cnt{F}{m}}
	\end{align}
\end{proposition}
The usage of prefix-slice of $h$ ensures monotonicity of the random variable, $\Cnt{F}{m}$, since from the definition of prefix-slice, 
we have that for every $1 \le m < n$, $h^{(m+1)}(y) = \alpha^{(m+1)} \Rightarrow h^{(m)}(y) = \alpha^{(m)}$. Formally,
\begin{proposition}
	For every $1 \le m < n$, $\Cell{F}{m+1} \subseteq \Cell{F}{m}$
\end{proposition}

\subsection{Helpful Combinatorial Inequality}

\begin{lemma}
	\label{lm: eta exponential}
	
Let $\eta(t, m, p) = \sum^t_{k=m} {t \choose k}p^k(1-p)^{t-k}$ and $p < 0.5$, then 
\begin{align*}
	\eta(t, \lceil t/2 \rceil, p) \in \Theta\left(t^{-\frac{1}{2}} \left(2\sqrt{p(1-p)}\right)^t\right)	
\end{align*}
\end{lemma}
\begin{proof}
	We will derive both an upper and a matching lower bound for $\eta(t, \lceil t/2 \rceil, p)$.
	We begin by deriving an upper bound:
	$\eta(t, \lceil t/2 \rceil, p) = \sum^t_{k=\lceil \frac{t}{2} \rceil} {t \choose k}p^k(1-p)^{t-k}$
	$\le {t \choose \lceil t/2 \rceil}\sum^t_{k=\lceil \frac{t}{2} \rceil}p^k(1-p)^{t-k}$
	$\le {t \choose \lceil t/2 \rceil} \cdot (p(1-p))^{\lceil \frac{t}{2} \rceil} \cdot \frac{1}{1-2p}$
	$\le \frac{1}{\sqrt{2\pi}} \cdot \frac{t}{\sqrt{\left(\frac{t}{2}-0.5\right)\left(\frac{t}{2}+0.5\right)}} \cdot \left(\frac{t}{t-1}\right)^t \cdot e^{\frac{1}{12t}-\frac{1}{6t+6}-\frac{1}{6t-6}} 
	\cdot t^{-\frac{1}{2}} 2^t \cdot (p(1-p))^{\frac{t}{2}} \cdot (p(1-p))^{\frac{1}{2}} \cdot \frac{1}{1-2p}$.
	The last inequality follows Stirling's approximation.
	As a result, $\eta(t, \lceil t/2 \rceil, p) \in \bigO{t^{-\frac{1}{2}} \left(2\sqrt{p(1-p)}\right)^t}$. 
	Afterwards; we move on to deriving a matching lower bound:
	$\eta(t, \lceil t/2 \rceil, p) = \sum^t_{k=\lceil \frac{t}{2} \rceil} {t \choose k}p^k(1-p)^{t-k}$
	$\ge {t \choose \lceil t/2 \rceil}p^{\lceil \frac{t}{2} \rceil}(1-p)^{t-\lceil \frac{t}{2} \rceil}$
	$\ge \frac{1}{\sqrt{2\pi}} \cdot \frac{t}{\sqrt{\left(\frac{t}{2}-0.5\right)\left(\frac{t}{2}+0.5\right)}} \cdot \left(\frac{t}{t+1}\right)^t \cdot e^{\frac{1}{12t}-\frac{1}{6t+6}-\frac{1}{6t-6}}
	\cdot t^{-\frac{1}{2}} 2^t \cdot (p(1-p))^{\frac{t}{2}} \cdot p^{\frac{1}{2}} (1-p)^{-\frac{1}{2}} \cdot \frac{1}{1-2p}$.
	The last inequality again follows Stirling's approximation.
	Hence, $\eta(t, \lceil t/2 \rceil, p) \in \Omega\left(t^{-\frac{1}{2}} \left(2\sqrt{p(1-p)}\right)^t\right)$.
	Combining these two bounds, we conclude that $\eta(t, \lceil t/2 \rceil, p) \in \Theta\left(t^{-\frac{1}{2}} \left(2\sqrt{p(1-p)}\right)^t\right)$.
	\qed
\end{proof}

%% file: sec/related_work.tex
\section{Related Work} \label{sec: related work}
The seminal work of Valiant established that {\sharpSAT} is {\sharpP}-complete~\cite{V79}. 
Toda later showed that every problem in the polynomial hierarchy could be solved by just a polynomial number of calls to a {\sharpP} oracle~\cite{T89}. 
Based on Carter and Wegman's seminal work on universal hash functions~\cite{CW77}, Stockmeyer proposed a probabilistic polynomial time procedure, with access to an {\NP} oracle, to obtain an $(\varepsilon, \delta)$-approximation of $F$~\cite{S83}.

Built on top of Stockmeyer's work, the core theoretical idea behind the hashing-based approximate solution counting framework, as presented in Algorithm~\ref{alg: appmc} ({\ApproxMC}~\cite{CMV13}), 
is to use 2-universal hash functions to partition the solution space (denoted by $\satisfying{F}$ for a given formula $F$) into \emph{small} cells of \emph{roughly equal} size. 
A cell is considered \emph{small} if the number of solutions it contains is less than or equal to a pre-determined threshold, {\thresh}. 
An {\NP} oracle is used to determine if a cell is small by iteratively enumerating its solutions until either there are no more solutions or $\thresh+1$ solutions have been found.
In practice, an SAT solver is used to implement the {\NP} oracle. 
To ensure a polynomial number of calls to the oracle, the threshold, {\thresh}, is set to be polynomial in the input parameter $\varepsilon$ at Line~\ref{ln: thresh}. 
The subroutine {\ApproxMCCore} takes the formula $F$ and {\thresh} as inputs and estimates the number of solutions at Line~\ref{ln: invoke core}. 
To determine the appropriate number of cells, i.e., the value of $m$ for $\hashspace(n, m)$, {\ApproxMCCore} uses a search procedure at Line~\ref{ln: estimate m} of Algorithm~\ref{alg: appmc core}. 
The estimate is calculated as the number of solutions in a randomly chosen cell, scaled by the number of cells, i.e., $2^m$ at Line~\ref{ln: core return}. 
To improve confidence in the estimate, {\ApproxMC} performs multiple runs of the {\ApproxMCCore} subroutine at Lines~\ref{ln: repeat begin}--~\ref{ln: repeat end} of Algorithm~\ref{alg: appmc}. 
The final count is computed as the median of the estimates obtained at Line~\ref{ln: find median}.

In the second version of {\ApproxMC}~\cite{CMV16}, two key algorithmic improvements are proposed to improve the practical performance by reducing the number of calls to the SAT solver. 
The first improvement is using galloping search to more efficiently find the correct number of cells, i.e., {\FibBinSearch} at Line~\ref{ln: log search} of Algorithm~\ref{alg: appmc core}.
The second is using linear search over a small interval around the previous value of $m$ before resorting to the galloping search. 
Additionally, the third and fourth versions~\cite{SM19,SGM20} enhance the algorithm's performance by effectively dealing with CNF formulas conjuncted with XOR constraints, commonly used in the hashing-based counting framework.
Moreover, an effective preprocessor named {\Arjun}~\cite{SM22} is proposed to enhance {\ApproxMC}'s performance by constructing shorter XOR constraints. 
As a result, the combination of {\Arjun} and {\ApproxMC} solved almost all existing benchmarks~\cite{SM22}, making it the current state of the art in this field.

In this work, we aim to address the main limitation of the {\ApproxMC} algorithm by focusing on an aspect that still needs to be improved upon by previous developments. Specifically, we aim to improve the core algorithm of {\ApproxMC}, which has remained unchanged.

%% file: sec/weakness.tex
\section{Weakness of {\ApproxMC}}
\label{sec: weakness}

\begin{algorithm}[t]
	\caption{\ApproxMC$(F, \varepsilon, \delta)$}
	\label{alg: appmc}
	
	\begin{algorithmic}[1]
		\State $\thresh \leftarrow 9.84 \left(1+\frac{\varepsilon}{1+\varepsilon}\right)\left(1+\frac{1}{\varepsilon}\right)^2;$ \label{ln: thresh}
		\State $Y \leftarrow \BoundedSAT(F, \thresh);$ \label{ln: appmc bounded sat}
		\If{$(|Y| < \thresh)$} \Return $|Y|;$ \EndIf
		\State $t \leftarrow \left\lceil 17\log_2(3/\delta)\right\rceil;$
		$C\leftarrow\emptyList; \iter \leftarrow 0;$ \label{ln: init}
		\Repeat	\label{ln: repeat begin}
		\State $\iter \leftarrow \iter + 1;$
		\State $\solCount \leftarrow \ApproxMCCore(F, \thresh);$ \label{ln: invoke core}
		\State $\AddToList(C,\solCount);$
		\Until{$(\iter \ge t)$}$;$ \label{ln: repeat end}
		\State $\finalestimate \leftarrow \FindMedian(C);$ \label{ln: find median}
		\State \Return $\finalestimate;$
		
	\end{algorithmic}
\end{algorithm}

\begin{algorithm}[t]
	\caption{\ApproxMCCore$(F,\thresh)$}
	\label{alg: appmc core}
	
	\begin{algorithmic}[1]
		\State Choose $h$ at random from $\hashspace(n,n);$
		\State Choose $\alpha$ at random from $\{0,1\}^n;$
		\State $m \leftarrow \FibBinSearch(F, h, \alpha, \thresh);$ \label{ln: estimate m} \label{ln: log search}
		\State $\Cnt{F}{m} \leftarrow \BoundedSAT\left(F\wedge \left(h^{(m)}\right)^{-1}\left(\alpha^{(m)}\right), \thresh\right);$ \label{ln: appmc core bounded sat}
		\State \Return $(2^m\times \Cnt{F}{m});$ \label{ln: core return}
	\end{algorithmic}
\end{algorithm}

As noted above, the core algorithm of {\ApproxMC} has not changed since 2016, and in this work, we aim to address the core limitation of {\ApproxMC}. To put our contribution in context, we first review {\ApproxMC} and its core algorithm, called {\ApproxMCCore}. 
We present the pseudocode of {\ApproxMC} and {\ApproxMCCore} in Algorithm~\ref{alg: appmc} and \ref{alg: appmc core}, respectively. 
 {\ApproxMCCore} may return an estimate that falls outside the {\PAC} range $\left[ \frac{|\satisfying{F}|}{1+\varepsilon}, (1 + \varepsilon)|\satisfying{F}| \right]$ with a certain probability of error. Therefore, {\ApproxMC} repeatedly invokes {\ApproxMCCore} (Lines~\ref{ln: repeat begin}--~\ref{ln: repeat end}) and returns the median of the estimates returned by {\ApproxMCCore} (Line~\ref{ln: find median}), which reduces the error probability to the user-provided parameter $\delta$.

 Let $\Obad_t$ denote the event that the median of $t$ estimates falls outside $\left[ \frac{|\satisfying{F}|}{1+\varepsilon}, (1 + \varepsilon)|\satisfying{F}| \right]$. Let $L$ denote the event that an invocation {\ApproxMCCore} returns an estimate less than $\frac{|\satisfying{F}|}{1+\varepsilon}$. Similarly, let $U$ denote the event that an individual estimate of $|\satisfying{F}|$ is greater than $(1 + \varepsilon)|\satisfying{F}|$. For simplicity of exposition, we assume $t$ is odd; the current implementation of $t$ indeed ensures that $t$ is odd by choosing the smallest odd $t$ for which $\Pr[\Obad_t] \leq \delta$.  

In the remainder of the section, we will demonstrate that reducing $\max\left\{\Prb{L}, \Prb{U}\right\}$ can effectively reduce the number of repetitions $t$, making the small-$\delta$ scenarios practical. To this end, we will first demonstrate the existing analysis technique of {\ApproxMC} leads to loose bounds on  $\Pr[\Obad_t]$. We then present a new analysis that leads to tighter bounds on $\Pr[\Obad_t]$.

The existing combinatorial analysis in~\cite{CMV13} derives the following proposition:
\begin{proposition}
	\label{prop:eq-bound}
\begin{align*}	
	\Prb{\Obad_t} \le \eta(t, \lceil t/2 \rceil, \Prb{L \cup U})
\end{align*} 
where $\eta(t, m, p) = \sum^t_{k=m} {t \choose k}p^k(1-p)^{t-k}$.
\end{proposition}
Proposition~\ref{prop:eq-bound} follows from the observation that if the median falls outside the {\PAC} range, 
at least $\left\lceil t/2 \right\rceil$ of the results must also be outside the range.
Let $\eta(t, \lceil t/2 \rceil, \Prb{L \cup U}) \le \delta$, and we can compute a valid $t$ at Line~\ref{ln: init} of {\ApproxMC}.

Proposition~\ref{prop:eq-bound} raises a question: can we derive a tight upper bound for $\Prb{\Obad_t}$?
The following lemma provides an affirmative answer to this question. 
\begin{lemma}
	\label{thm: error bound}
	Assuming $t$ is odd, we have:
	\begin{align*}
		\Prb{\Obad_t} = \eta(t, \lceil t/2 \rceil, \Prb{L}) + \eta(t, \lceil t/2 \rceil, \Prb{U})
	\end{align*}
\end{lemma}

\begin{proof}
	Let $I^L_i$ be an indicator variable that is 1 when {\ApproxMCCore} returns a $\solCount$ less than $ \frac{|\satisfying{F}|}{1+\varepsilon}$, indicating the occurrence of event $L$ in the $i$-th repetition.
	Let $I^U_i$ be an indicator variable that is 1 when {\ApproxMCCore} returns a $\solCount$ greater than $(1+\varepsilon)|\satisfying{F}|$, indicating the occurrence of event $U$ in the $i$-th repetition.
	We aim first to prove that $\Obad_t \Leftrightarrow \left(\sum_{i=1}^{t} I_i^L \ge \left\lceil \frac{t}{2} \right\rceil\right) \vee \left(\sum_{i=1}^{t} I_i^U \ge \left\lceil \frac{t}{2} \right\rceil\right)$. 
	We will begin by proving the right ($\Rightarrow$) implication. 
	If the median of $t$ estimates violates the {\PAC} guarantee, the median is either less than $\frac{|\satisfying{F}|}{1+\varepsilon}$ or greater than $(1+\varepsilon)|\satisfying{F}|$.
	In the first case, since half of the estimates are less than the median, at least $\left\lceil \frac{t}{2} \right\rceil$ estimates are less than $\frac{|\satisfying{F}|}{1+\varepsilon}$.
	Formally, this implies $\sum_{i=1}^{t} I_i^L \ge \left\lceil \frac{t}{2} \right\rceil$. 
	Similarly, in the case that the median is greater than $(1+\varepsilon)|\satisfying{F}|$, since half of the estimates are greater than the median, at least $\left\lceil \frac{t}{2} \right\rceil$ estimates are greater than $(1+\varepsilon)|\satisfying{F}|$,
	thus formally implying $\sum_{i=1}^{t} I_i^U \ge \left\lceil \frac{t}{2} \right\rceil$.
	On the other hand, we prove the left $(\Leftarrow)$ implication. Given $\sum_{i=1}^{t} I_i^L \ge \left\lceil \frac{t}{2} \right\rceil$, more than half of the estimates are less than $\frac{|\satisfying{F}|}{1+\varepsilon}$, 
	and therefore the median is less than $\frac{|\satisfying{F}|}{1+\varepsilon}$, violating the {\PAC} guarantee. 
	Similarly, given $\sum_{i=1}^{t} I_i^U \ge \left\lceil \frac{t}{2} \right\rceil$, more than half of the estimates are greater than $(1+\varepsilon)|\satisfying{F}|$,
	and therefore the median is greater than $(1+\varepsilon)|\satisfying{F}|$, violating the {\PAC} guarantee.
	This concludes the proof of $\Obad_t \Leftrightarrow \left(\sum_{i=1}^{t} I_i^L \ge \left\lceil \frac{t}{2} \right\rceil\right) \vee \left(\sum_{i=1}^{t} I_i^U \ge \left\lceil \frac{t}{2} \right\rceil\right)$.
	Then we obtain:
	\begin{align*}
		\Prb{\Obad_t} &= \Prb{\left(\sum_{i=1}^{t} I_i^L \ge \left\lceil t/2 \right\rceil\right) \vee \left(\sum_{i=1}^{t} I_i^U \ge \left\lceil t/2 \right\rceil\right)} \\
		&= \Prb{\left(\sum_{i=1}^{t} I_i^L \ge \left\lceil t/2 \right\rceil\right)} + \Prb{\left(\sum_{i=1}^{t} I_i^U \ge \left\lceil t/2 \right\rceil\right)} \\
		&- \Prb{\left(\sum_{i=1}^{t} I_i^L \ge \left\lceil t/2 \right\rceil\right) \wedge \left(\sum_{i=1}^{t} I_i^U \ge \left\lceil t/2 \right\rceil\right)}
	\end{align*}
	Given $I^L_i+I^U_i \le 1$ for $i=1, 2, ..., t$, $\sum_{i=1}^{t} (I^L_i + I^U_i) \le t$ is there, 
	but if $\left(\sum_{i=1}^{t} I_i^L \ge \left\lceil t/2 \right\rceil\right) \wedge \left(\sum_{i=1}^{t} I_i^U \ge \left\lceil t/2 \right\rceil\right)$ is also given, we obtain $\sum_{i=1}^{t} (I^L_i + I^U_i) \ge t+1$ 
	contradicting $\sum_{i=1}^{t} (I^L_i + I^U_i) \le t$; 
	Hence, we can conclude that $\Prb{\left(\sum_{i=1}^{t} I_i^L \ge \left\lceil t/2 \right\rceil\right) \wedge \left(\sum_{i=1}^{t} I_i^U \ge \left\lceil t/2 \right\rceil\right)}=0$. From this, we can deduce:
	\begin{align*}
		\Prb{\Obad_t} &= \Prb{\left(\sum_{i=1}^{t} I_i^L \ge \left\lceil t/2 \right\rceil\right)} + \Prb{\left(\sum_{i=1}^{t} I_i^U \ge \left\lceil t/2 \right\rceil\right)} \\
		&= \eta(t, \lceil t/2 \rceil, \Prb{L}) + \eta(t, \lceil t/2 \rceil, \Prb{U})
	\end{align*}
	\qed
\end{proof}

Though~\Cref{thm: error bound} shows that reducing $\Prb{L}$ and $\Prb{U}$ can decrease the error probability,
it is still uncertain to what extent $\Prb{L}$ and $\Prb{U}$ affect the error probability.
To further understand this impact, the following lemma is presented to establish a correlation between the error probability and $t$ depending on $\Prb{L}$ and $\Prb{U}$.

\begin{lemma}
	\label{thm: error prob max}
	Let $p_{max} = \max\left\{\Prb{L}, \Prb{U}\right\}$ and $p_{max} < 0.5$, we have 
	\begin{align*}
		\Prb{\Obad_t} \in \Theta\left(t^{-\frac{1}{2}} \left(2\sqrt{p_{max}(1-p_{max})}\right)^t\right)
	\end{align*}
\end{lemma}
\begin{proof}
	Applying Lemma~\ref{lm: eta exponential} and \ref{thm: error bound}, we have 
	\begin{align*}
		\Prb{\Obad_t} &\in \Theta\left(t^{-\frac{1}{2}} \left( \left(2\sqrt{\Prb{L}(1-\Prb{L})}\right)^t + \left(2\sqrt{\Prb{U}(1-\Prb{U})}\right)^t\right) \right) \\
		&= \Theta\left(t^{-\frac{1}{2}} \left(2\sqrt{p_{max}(1-p_{max})}\right)^t\right)
	\end{align*}
	\qed
\end{proof}

In summary, \Cref{thm: error prob max} provides a way to tighten the bound on $\Pr[\Obad_t]$ by designing an algorithm such that we can obtain a tighter bound on $p_{max}$ in contrast to previous approaches that relied on obtaining a tighter bound on $\Pr[L \cup U]$.

%% file: sec/theory.tex
\section{Rounding Model Counting} \label{sec: theory}
In this section, we present a \emph{rounding}-based technique that allows us to obtain a tighter bound on $p_{max}$. On a high-level, instead of returning the estimate from one iteration of the underlying core algorithm as the number of solutions in a randomly chosen cell multiplied by the number of cells, we {\em round} each estimate of the model count to a value that is more likely to be within $(1+\varepsilon)$-bound. While counter-intuitive at first glance,
we show that rounding the estimate reduces $\max\left\{\Prb{L}, \Prb{U}\right\}$, thereby resulting in a smaller number of repetitions of the underlying algorithm. 

We present {\ApproxMCSix}, a \emph{rounding}-based approximate model counting algorithm, in Section~\ref{sec: algorithm}. 
Section~\ref{sec: analysis} will demonstrate how {\ApproxMCSix} decreases $\max\left\{\Prb{L}, \Prb{U}\right\}$ and the number of estimates. 
Lastly, in Section~\ref{sec: proof one case}, we will provide proof of the theoretical correctness of the algorithm.

\subsection{Algorithm} \label{sec: algorithm}
Algorithm~\ref{alg: roundmc} presents the procedure of {\ApproxMCSix}. 
{\ApproxMCSix} takes as input a formula $F$, a tolerance parameter $\varepsilon$, and a confidence parameter $\delta$.
{\ApproxMCSix} returns an $(\varepsilon, \delta)$-estimate $c$ of $|\satisfying{F}|$ such that $\Prob{\frac{|\satisfying{F}|}{1+\varepsilon} \le c \le (1 + \varepsilon)|\satisfying{F}|} \ge 1 - \delta$.
{\ApproxMCSix} is identical to {\ApproxMC} in its initialization of data structures and handling of base cases (Lines~\ref{ln:roundmc-begin}--~\ref{ln: roundmc init iter}).

In Line~\ref{ln: roundmc config round}, we pre-compute the rounding type and rounding value to be used in {\ApproxMCSixCore}.
{\configround} is implemented in Algorithm~\ref{alg: config round}; the precise choices arise due to technical analysis, as presented in  Section~\ref{sec: analysis}.
Note that, in {\configround}, $\Cnt{F}{m}$ is \emph{rounded up} to {\roundvalue} for $\varepsilon<3$ ($\roundtype=1$) but \emph{rounded} to {\roundvalue} for $\varepsilon\ge3$ ($\roundtype=0$). 
Rounding up means $\Cnt{F}{m} = \roundvalue$ only if $\Cnt{F}{m} < \roundvalue$.
Rounding means $\Cnt{F}{m} = \roundvalue$ in all cases.
{\ApproxMCSix} computes the number of repetitions necessary to lower error probability down to $\delta$ at Line~\ref{ln: roundmc compute iter}. 
The implementation of {\computeIter} is presented in Algorithm~\ref{alg: compute iter} following \Cref{thm: error bound}.
The iterator keeps increasing until the tight error bound is no more than $\delta$.
As we will show in Section~\ref{sec: analysis}, $\Prb{L}$ and $\Prb{U}$ depend on $\varepsilon$.
In the loop of Lines~\ref{ln: roundmc repeat begin}--~\ref{ln: roundmc repeat end}, {\ApproxMCSixCore} repeatedly estimates $|\satisfying{F}|$.
Each estimate $\solCount$ is stored in List $C$, and the median of $C$ serves as the final estimate satisfying the $(\varepsilon,\delta)$-guarantee.

Algorithm~\ref{alg: roundmc core} shows the pseudo-code of {\ApproxMCSixCore}.
A random hash function is chosen at Line~\ref{ln: roundmc hash} to partition $\satisfying{F}$ into \emph{roughly equal} cells.
A random hash value is chosen at Line~\ref{ln: roundmc hash value} to randomly pick a cell for estimation.
In Line~\ref{ln: roundmc search m}, we search for a value $m$ such that the cell picked from $2^m$ available cells is \emph{small} enough to enumerate solutions one by one while providing a good estimate of $|\satisfying{F}|$.
In Line~\ref{ln: roundmc count cell}, a bounded model counting is invoked to compute the size of the picked cell, i.e., $\Cnt{F}{m}$.
Finally, if {\roundtype} equals $1$, $\Cnt{F}{m}$ is rounded up to {\roundvalue} at Line \ref{ln: roundmc round up}. 
Otherwise, {\roundtype} equals $0$, and $\Cnt{F}{m}$ is rounded to {\roundvalue} at Line \ref{ln: roundmc round}.
Note that \emph{rounding up} returns {\roundvalue} only if $\Cnt{F}{m}$ is less than {\roundvalue}. 
However, in the case of \emph{rounding}, {\roundvalue} is always returned no matter what value {\Cnt{F}{m}} is.

\begin{algorithm}[t]
	\caption{\ApproxMCSix$(F, \varepsilon, \delta)$}
	\label{alg: roundmc}
	
	\begin{algorithmic}[1]
		\State $\thresh \leftarrow 9.84\label{ln:roundmc-begin} \left(1+\frac{\varepsilon}{1+\varepsilon}\right)\left(1+\frac{1}{\varepsilon}\right)^2;$ \label{ln: roundmc thresh}
		\State $Y \leftarrow \BoundedSAT(F, \thresh);$ \label{ln: roundmc precheck thresh}
		\If{$(|Y| < \thresh)$} \Return $|Y|;$ \EndIf

		\State $C\leftarrow\emptyList; \iter \leftarrow 0;$ \label{ln: roundmc init iter}
		\State (\roundtype, \roundvalue) $\leftarrow \configround(\varepsilon)$ \label{ln: roundmc config round}
		\State $t \leftarrow \computeIter(\varepsilon, \delta)$ \label{ln: roundmc compute iter}
		\Repeat \label{ln: roundmc repeat begin}
		\State $\iter \leftarrow \iter + 1;$
		\State $\solCount \leftarrow \ApproxMCSixCore(F, \thresh, \roundtype, \roundvalue);$ \label{ln: roundmc core}
		\State $\AddToList(C,\solCount);$
		\Until{$(\iter \ge t)$}$;$ \label{ln: roundmc repeat end}
		\State finalEstimate $\leftarrow \FindMedian(C);$
		\State \Return finalEstimate $;$
		
	\end{algorithmic}
\end{algorithm}
\begin{algorithm}[t]
	\caption{\ApproxMCSixCore$(F,\thresh,\roundtype,\roundvalue)$}
	\label{alg: roundmc core}
	
	\begin{algorithmic}[1]
		\State Choose $h$ at random from $\hashspace(n,n);$ \label{ln: roundmc hash}
		\State Choose $\alpha$ at random from $\{0,1\}^n;$ \label{ln: roundmc hash value}
		\State $m \leftarrow \FibBinSearch(F, h, \alpha, \thresh);$ \label{ln: roundmc search m}
		\State $\Cnt{F}{m} \leftarrow \BoundedSAT\left(F\wedge \left(h^{(m)}\right)^{-1}\left(\alpha^{(m)}\right), \thresh\right);$ \label{ln: roundmc count cell}
		\If{{\roundtype} = 1}
		\State \Return $(2^m\times \max\{\Cnt{F}{m},\roundvalue\});$ \label{ln: roundmc round up}
		\Else
		\State \Return $(2^m\times \roundvalue);$ \label{ln: roundmc round}
		\EndIf
	\end{algorithmic}
\end{algorithm}

\begin{algorithm}[t]
	\caption{$\configround(\varepsilon)$}
	\label{alg: config round}
	
	\begin{algorithmic}[1]
		\If{$(\varepsilon<\sqrt{2}-1)$} \Return $(1, \frac{\sqrt{1+2\varepsilon}}{2}\pivot);$
		\ElsIf{$(\varepsilon<1)$} \Return $(1, \frac{\pivot}{\sqrt{2}});$
		\ElsIf{$(\varepsilon<3)$} \Return $(1, \pivot);$
		\ElsIf{$(\varepsilon<4\sqrt{2}-1)$} \Return $(0, \pivot);$
		\Else \State \Return $(0, \sqrt{2}\pivot);$
		\EndIf
	\end{algorithmic}
\end{algorithm}

\begin{algorithm}[t]
	\caption{$\computeIter(\varepsilon,\delta)$}
	\label{alg: compute iter}
	
	\begin{algorithmic}[1]
		\State $\iter \leftarrow 1;$ \label{ln: compute iter init}
		\While{$(\eta(\iter, \lceil \iter/2 \rceil, \Pr_\varepsilon[L]) + \eta(\iter, \lceil \iter/2 \rceil, \Pr_\varepsilon[U]) > \delta)$} \label{ln: compute iter loop start}
		\State $\iter \leftarrow \iter + 2;$ \label{ln: compute iter loop end}
		\EndWhile	
		\State \Return \iter$;$
	\end{algorithmic}
\end{algorithm}

\subsection{Repetition Reduction}
\label{sec: analysis}
We will now show that {\ApproxMCSixCore} allows us to obtain a smaller $\max\left\{\Prb{L}, \Prb{U}\right\}$.
Furthermore, we show the large gap between the error probability of {\ApproxMCSix} and that of {\ApproxMC} both analytically and visually.

The following lemma presents the upper bounds of $\Prb{L}$ and $\Prb{U}$ for {\ApproxMCSixCore}. 
Let $\pivot=9.84\left(1+\frac{1}{\varepsilon}\right)^2$ for simplicity.
\begin{restatable}{lemma}{LUbound}
		\label{thm: complete L U bound}
		The following bounds hold for {\ApproxMCSix}:
		\begin{align*}
			\Prb{L} \le
			\begin{cases}
				0.262 & \text{if } \varepsilon<\sqrt{2}-1 \\
				0.157 & \text{if } \sqrt{2}-1\le\varepsilon<1 \\
				0.085 & \text{if } 1\le\varepsilon<3 \\
				0.055 & \text{if } 3\le\varepsilon<4\sqrt{2}-1 \\
				0.023 & \text{if } \varepsilon\ge4\sqrt{2}-1 \\
			\end{cases}
		\end{align*}
		\begin{align*}
			\Prb{U} \le
			\begin{cases}
				0.169 & \text{if } \varepsilon<3 \\
				0.044 & \text{if } \varepsilon\ge3 \\
			\end{cases}
		\end{align*}
\end{restatable}
The proof of Lemma~\ref{thm: complete L U bound} is deferred to Section~\ref{sec: proof one case}. 
Observe that Lemma~\ref{thm: complete L U bound} influences the choices in the design of {\configround} (Algorithm~\ref{alg: config round}).
Recall that $\max\left\{\Prb{L}, \Prb{U}\right\} \le 0.36$ for {\ApproxMC} (\Cref{sec: appendix appmc}), but
Lemma~\ref{thm: complete L U bound} ensures $\max\left\{\Prb{L}, \Prb{U}\right\} \le 0.262$ for {\ApproxMCSix}.
For $\varepsilon \ge 4\sqrt{2}-1$, Lemma~\ref{thm: complete L U bound} even delivers $\max\left\{\Prb{L}, \Prb{U}\right\} \le 0.044$.

The following theorem analytically presents the gap between the error probability of {\ApproxMCSix} and that of {\ApproxMC}\footnote{
We state the result for the case $\sqrt{2}-1\le\varepsilon<1$. 
A similar analysis can be applied to other cases, which leads to an even bigger gap between {\ApproxMCSix} and {\ApproxMC}.
}. 
\begin{theorem}
	\label{thm: exp gap}
	For $\sqrt{2}-1\le\varepsilon<1$,
	\begin{align*}
		\Prb{\Obad_t} \in
		\begin{cases}
			\bigO{t^{-\frac{1}{2}}0.75^t} & \text{for {\ApproxMCSix}} \\
			\bigO{t^{-\frac{1}{2}}0.96^t} & \text{for {\ApproxMC}}\\
		\end{cases}
	\end{align*}
\end{theorem}
\begin{proof}
	From \Cref{thm: complete L U bound}, we obtain $p_{max} \le 0.169$ for {\ApproxMCSix}.
	Applying \Cref{thm: error prob max}, we have 
	\begin{align*}
		\Prb{\Obad_t} \in \bigO{t^{-\frac{1}{2}} \left(2\sqrt{0.169(1-0.169)}\right)^t} \subseteq \bigO{t^{-\frac{1}{2}}0.75^t}
	\end{align*}
	For {\ApproxMC}, combining $p_{max} \le 0.36$ (\Cref{sec: appendix appmc}) and \Cref{thm: error prob max}, 
	we obtain 
	\begin{align*}
		\Prb{\Obad_t} \in \bigO{t^{-\frac{1}{2}} \left(2\sqrt{0.36(1-0.36)}\right)^t} = \bigO{t^{-\frac{1}{2}}0.96^t}
	\end{align*}
	\qed
\end{proof}
Figure~\ref{fig: complete error bound} visualizes the large gap between the error probability of {\ApproxMCSix} and that of {\ApproxMC}.
The x-axis represents the number of repetitions ($t$) in {\ApproxMCSix} or {\ApproxMC}.
The y-axis represents the upper bound of error probability in the log scale. 
For example, as $t=117$, {\ApproxMC} guarantees that with a probability of $10^{-3}$, the median over 117 estimates violates the {\PAC} guarantee.
However, {\ApproxMCSix} allows a much smaller error probability that is at most $10^{-15}$ for $\sqrt{2}-1\le\varepsilon<1$.
The smaller error probability enables {\ApproxMCSix} to repeat fewer repetitions while providing the same level of theoretical guarantee.
For example, given $\delta=0.001$ to {\ApproxMC}, i.e., $y=0.001$ in \Cref{fig: complete error bound}, {\ApproxMC} requests 117 repetitions to obtain the given error probability. 
However, {\ApproxMCSix} claims that 37 repetitions for $\varepsilon<\sqrt{2}-1$, 19 repetitions for $\sqrt{2}-1\le\varepsilon<1$, 17 repetitions for $1\le\varepsilon<3$,
7 repetitions for $3\le\varepsilon<4\sqrt{2}-1$, and 5 repetitions for $\varepsilon\ge4\sqrt{2}-1$ are sufficient to obtain the same level of error probability.
Consequently, {\ApproxMCSix} can obtain $3\times$, $6\times$, $7\times$, $17\times$, and $23\times$ speedups, respectively, than {\ApproxMC}.

\begin{figure}[h]
	\centering
	\includegraphics[scale=0.55]{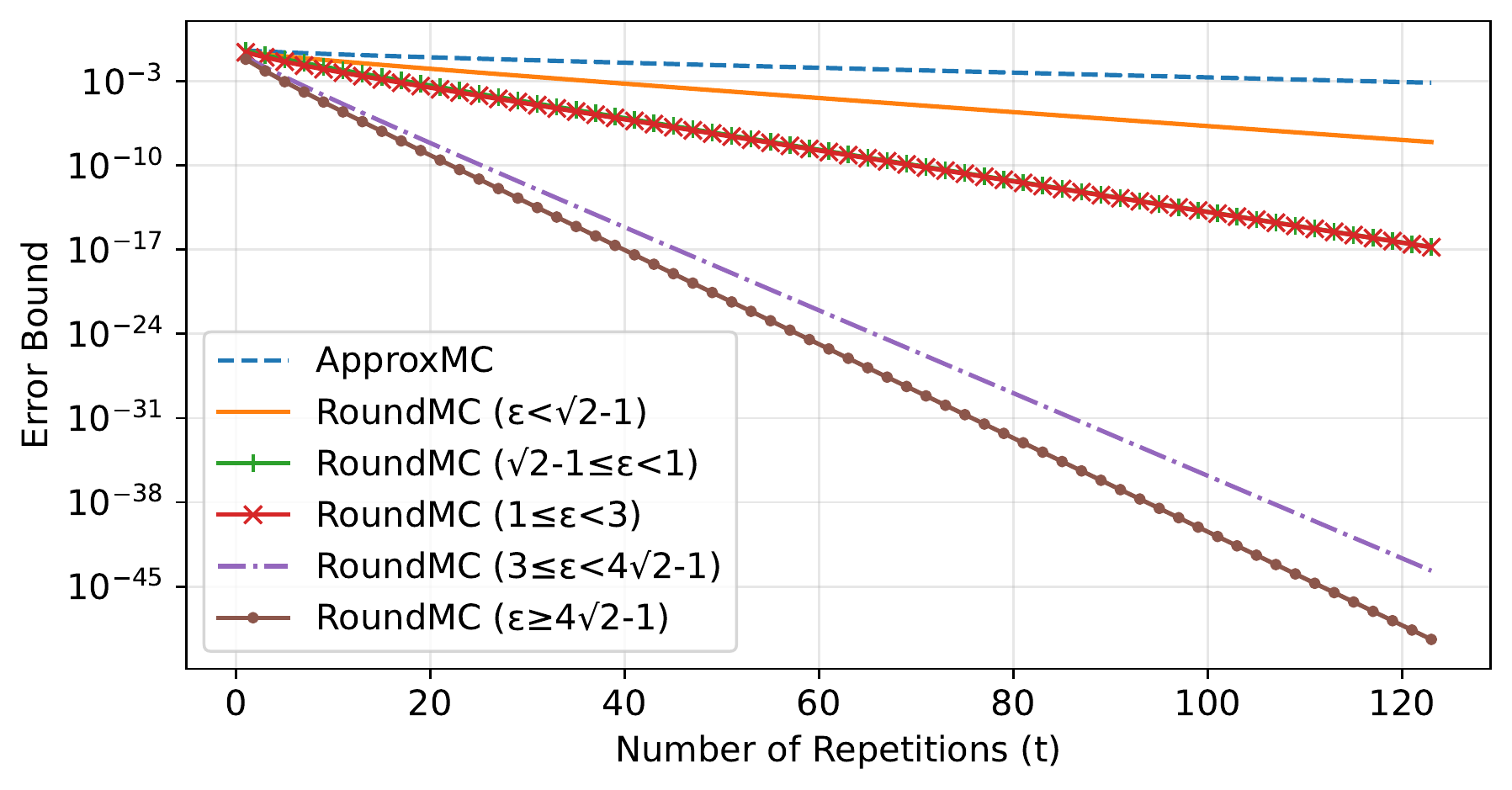}
	\caption{Comparison of error bounds for {\ApproxMCSix} and {\ApproxMC}.}
	\label{fig: complete error bound}
\end{figure}

\subsection{Proof of Lemma~\ref{thm: complete L U bound} for case  $\sqrt{2}-1\le\varepsilon<1$}
\label{sec: proof one case}

We provide full proof of Lemma~\ref{thm: complete L U bound} for case  $\sqrt{2}-1\le\varepsilon<1$. We defer the proof of other cases to \Cref{sec: appendix roundmc}.

Let $T_m$ denote the event $\left(\Cnt{F}{m} < \thresh \right)$, and let $L_m$ and $U_m$ denote the events $\left( \Cnt{F}{m} < \frac{\Exp{\Cnt{F}{m}}}{1+\varepsilon} \right)$
and $\left( \Cnt{F}{m} > \Exp{\Cnt{F}{m}}(1+\varepsilon) \right)$, respectively.
To ease the proof, let $U'_m$ denote $\left( \Cnt{F}{m} > \Exp{\Cnt{F}{m}}(1+\frac{\varepsilon}{1+\varepsilon}) \right)$, and thereby $U_m \subseteq U'_m$.
Let $m^\ast = \left\lfloor \log_2|\satisfying{F}| - \log_2\left(\pivot\right) + 1 \right\rfloor$ such that $m^\ast$ is the smallest $m$ satisfying $\frac{|\satisfying{F}|}{2^m}(1+\frac{\varepsilon}{1+\varepsilon}) \le \thresh-1$.

Let us first prove the lemmas used in the proof of Lemma~\ref{thm: complete L U bound}.

\begin{lemma}
	\label{lm: inequality}
	For every $0<\beta<1$, $\gamma>1$, and $1 \le m \le n$, the following holds:
	\begin{enumerate}
		\item $\Prb{\Cnt{F}{m} \le \beta\Exp{\Cnt{F}{m}}} \le \frac{1}{1+(1-\beta)^2\Exp{\Cnt{F}{m}}}$
		\item $\Prb{\Cnt{F}{m} \ge \gamma\Exp{\Cnt{F}{m}}} \le \frac{1}{1+(\gamma-1)^2\Exp{\Cnt{F}{m}}}$
	\end{enumerate}
\end{lemma}

\begin{proof}
	Statement 1 can be proved following the proof of Lemma 1 in~\cite{CMV16}.$\;$ For statement 2, we rewrite the left-hand side and apply Cantelli's inequality:
	$\Prb{\Cnt{F}{m} - \Exp{\Cnt{F}{m}} \ge (\gamma-1)\Exp{\Cnt{F}{m}}} \le \frac{\sigma^2\left[ \Cnt{F}{m} \right]}{\sigma^2\left[ \Cnt{F}{m} \right]+ ((\gamma-1)\Exp{\Cnt{F}{m}})^2}$.
	Finally, applying Equation~\ref{eq: var exp} completes the proof.
	\qed
\end{proof}

\begin{lemma}
	\label{lm: bound}
	Given $\sqrt{2}-1\le\varepsilon<1$, the following bounds hold: 
	\begin{enumerate}
		\item $\Prb{T_{m^\ast-3}} \le \frac{1}{62.5}$
		\item $\Prb{L_{m^\ast-2}} \le \frac{1}{20.68}$
		\item $\Prb{L_{m^\ast-1}} \le \frac{1}{10.84}$
		\item $\Prb{U'_{m^\ast}} \le \frac{1}{5.92}$
	\end{enumerate}
\end{lemma}

\begin{proof}
	Following the proof of Lemma 2 in~\cite{CMV16}, we can prove statements 1, 2, and 3. 
	To prove statement 4, replacing $\gamma$ with $(1+\frac{\varepsilon}{1+\varepsilon})$ in Lemma~\ref{lm: inequality}
	and employing $\Exp{\Cnt{F}{m^\ast}} \ge \pivot/2$, we obtain $\Prb{U'_{m^\ast}} \le \frac{1}{1+\left(\frac{\varepsilon}{1+\varepsilon}\right)^2\pivot/2} \le \frac{1}{5.92}$.
	\qed
\end{proof}

Now we prove the upper bounds of $\Prb{L}$ and $\Prb{U}$ in Lemma~\ref{thm: complete L U bound} for $\sqrt{2}-1\le\varepsilon<1$.
The proof for other $\varepsilon$ is deferred to \Cref{sec: appendix roundmc} due to the page limit.

\LUbound*

\begin{proof}
	We prove the case of $\sqrt{2}-1\le\varepsilon<1$.
	The proof for other $\varepsilon$ is deferred to \Cref{sec: appendix roundmc}.
	Let us first bound $\Prb{L}$.
	Following {\FibBinSearch} in~\cite{CMV16}, we have
	\begin{align}
		\label{eq: L expansion}
		\Prb{L} = \left[ \bigcup_{i\in\{1,...,n\}} \left( \overline{T_{i-1}} \cap T_i \cap L_i \right) \right]
	\end{align}
	Equation~\ref{eq: L expansion} can be simplified by three observations labeled $O1, O2$ and $O3$ below.
	\begin{description}
		\item[$O1:$] $\forall i \le m^\ast-3, T_i \subseteq T_{i+1}$. Therefore,
		\begin{align*}
			\bigcup_{i\in\{1,...,m^\ast-3\}} (\overline{T_{i-1}} \cap T_i \cap L_i) \subseteq \bigcup_{i\in\{1,...,m^\ast-3\}} T_i \subseteq T_{m^\ast-3}
		\end{align*}
		\item[$O2:$] For $i \in \{m^\ast-2, m^\ast-1\}$, we have
		\begin{align*}
			\bigcup_{i\in\{m^\ast-2, m^\ast-1\}} (\overline{T_{i-1}} \cap T_i \cap L_i) \subseteq L_{m^\ast-2} \cup L_{m^\ast-1}
		\end{align*}
		\item[$O3:$] $\forall i \ge m^\ast$, since rounding $\Cnt{F}{i}$ up to $\frac{\pivot}{\sqrt{2}}$ and $m^\ast \ge \log_2|\satisfying{F}| - \log_2\left(\pivot\right)$, we have 
		$2^i \times \Cnt{F}{i} \ge 2^{m^\ast} \times \frac{\pivot}{\sqrt{2}} \ge \frac{|\satisfying{F}|}{\sqrt{2}} \ge \frac{|\satisfying{F}|}{1+\varepsilon}$. The last inequality follows from $\varepsilon \ge \sqrt{2}-1$.
		Then we have $\Cnt{F}{i} \ge \frac{\Exp{\Cnt{F}{i}}}{1+\varepsilon}$.
		Therefore, $L_i=\emptyset$ for $i \ge m^\ast$ and we have
		\begin{align*}
			\bigcup_{i\in\{m^\ast, ..., n\}} (\overline{T_{i-1}} \cap T_i \cap L_i) = \emptyset
		\end{align*}
	\end{description}
	Following the observations $O1$, $O2$, and $O3$, we simplify Equation~\ref{eq: L expansion} and obtain
	\begin{align*}
		\Prb{L} \le \Prb{T_{m^\ast-3}} + \Prb{L_{m^\ast-2}} + \Prb{L_{m^\ast-1}}
	\end{align*}
	Employing Lemma~\ref{lm: bound} gives $\Prb{L} \le 0.157$.
	
	Now let us bound $\Prb{U}$. Similarly, following {\FibBinSearch} in~\cite{CMV16}, we have
	\begin{align}
		\label{eq: U expansion}
		\Prb{U} = \left[ \bigcup_{i\in\{1,...,n\}} \left( \overline{T_{i-1}} \cap T_i \cap U_i \right) \right]
	\end{align}
	We derive the following observations $O4$ and $O5$.
	\begin{description}
		\item[$O4:$] $\forall i \le m^\ast-1$, since $m^\ast \le \log_2|\satisfying{F}| - \log_2\left(\pivot\right) + 1$, 
		we have $2^i \times \Cnt{F}{i} \le 2^{m^\ast-1} \times \thresh \le |\satisfying{F}|\left( 1+\frac{\varepsilon}{1+\varepsilon} \right)$.
		Then we obtain $\Cnt{F}{i} \le \Exp{\Cnt{F}{i}}\left( 1+\frac{\varepsilon}{1+\varepsilon} \right)$.
		Therefore, $T_i \cap U'_i = \emptyset$ for $i \le m^\ast-1$ and we have
		\begin{align*}
			\bigcup_{i\in\{1,...,m^\ast-1\}} \left( \overline{T_{i-1}} \cap T_i \cap U_i \right)
			\subseteq \bigcup_{i\in\{1,...,m^\ast-1\}} \left( \overline{T_{i-1}} \cap T_i \cap U'_i \right) = \emptyset
		\end{align*}
		\item[$O5:$] $\forall i \ge m^\ast$, $\overline{T_i}$ implies $\Cnt{F}{i} > \thresh$, and then we have
		$2^i \times \Cnt{F}{i} > 2^{m^\ast} \times \thresh \ge |\satisfying{F}|\left( 1+\frac{\varepsilon}{1+\varepsilon} \right)$.
		The second inequality follows from $m^\ast \ge \log_2|\satisfying{F}| - \log_2\left(\pivot\right)$.
		Then we obtain $\Cnt{F}{i} > \Exp{\Cnt{F}{i}}\left( 1+\frac{\varepsilon}{1+\varepsilon} \right)$.
		Therefore, $\overline{T_i} \subseteq U'_i$ for $i \ge m^\ast$. 
		Since $\forall i, \overline{T_i} \subseteq \overline{T_{i-1}}$, we have 
		\begin{align}
			\bigcup_{i\in\{m^\ast,...,n\}} \left( \overline{T_{i-1}} \cap T_i \cap U_i \right) 
			& \subseteq \bigcup_{i\in\{m^\ast+1,...,n\}} \overline{T_{i-1}} \cup ( \overline{T_{m^\ast-1}} \cap T_{m^\ast} \cap U_{m^\ast} ) \nonumber \\
			& \subseteq \overline{T_{m^\ast}} \cup ( \overline{T_{m^\ast-1}} \cap T_{m^\ast} \cap U_{m^\ast} ) \nonumber \\
			& \subseteq \overline{T_{m^\ast}} \cup U_{m^\ast} \nonumber \\
			& \subseteq U'_{m^\ast} \label{eq: round U}
		\end{align}
		Remark that for $\sqrt{2}-1\le\varepsilon<1$, we round $\Cnt{F}{m^\ast}$ up to $\frac{\pivot}{\sqrt{2}}$, 
		and we have $2^{m^\ast}\times\frac{\pivot}{\sqrt{2}} \le |\satisfying{F}|(1+\varepsilon)$, 
		which means \emph{rounding} doesn't affect the event $U_{m^\ast}$; therefore, Inequality~\ref{eq: round U} still holds.
	\end{description}
	Following the observations $O4$ and $O5$, we simplify Equation~\ref{eq: U expansion} and obtain
	\begin{align*}
		\Prb{U} \le \Prb{U'_{m^\ast}}
	\end{align*}
	Employing Lemma~\ref{lm: bound} gives $\Prb{U} \le 0.169$.
	\qed
\end{proof}

%% file: sec/experiment.tex
\section{Experimental Evaluation} \label{sec: experiment}
It is perhaps worth highlighting that both {\ApproxMCCore} and {\ApproxMCSixCore} invoke the underlying SAT solver on identical queries; the only difference between {\ApproxMCSix} and {\ApproxMC} lies in what estimate to return and how often {\ApproxMCCore} and {\ApproxMCSixCore} are invoked. From this viewpoint, one would expect that theoretical improvements would also lead to improved runtime performance. To provide further evidence,  we perform extensive empirical evaluation and compare {\ApproxMCSix}'s performance against the  current state-of-the-art model counter, {\ApproxMC}~\cite{SGM20}.  We use {\Arjun} as a pre-processing tool. We used the latest version of {\ApproxMC}, called {\ApproxMCFour}; an entry based on {\ApproxMCFour} won the Model Counting Competition 2022. 

Previous comparisons of {\ApproxMC} have been performed on a set of 1896 instances, but the latest version of {\ApproxMC} is able to solve almost all the instances when these instances are pre-processed by {\Arjun}. Therefore, we sought to construct a new comprehensive set of 1890 instances derived from various sources, 
including Model Counting Competitions 2020-2022~\cite{FHH20,HF21,HF22}, program synthesis~\cite{ABJM+13}, quantitative control improvisation~\cite{GVF22}, quantification of software properties~\cite{TW21}, and adaptive chosen ciphertext attacks~\cite{BZG20}. As noted earlier, our technique extends to projected model counting, and our benchmark suite indeed comprises 772 projected model counting instances.

Experiments were conducted on a high-performance computer cluster, with each node consisting of 2xE5-2690v3 CPUs featuring 2x12 real cores and 96GB of RAM.
For each instance, a counter was run on a single core, with a time limit of 5000 seconds and a memory limit of 4GB.
To compare runtime performance, we use the PAR-2 score, a standard metric in the SAT community.
Each instance is assigned a score that is the number of seconds it takes the corresponding tool to complete execution successfully. 
In the event of a timeout or memory out, the score is the doubled time limit in seconds.
The PAR-2 score is then calculated as the average of all the instance scores.
We also report the speedup of {\ApproxMCSix} over {\ApproxMCFour}, calculated as the ratio of the runtime of {\ApproxMCFour} to that of {\ApproxMCSix} on instances solved by both counters. We set $\delta$ to 0.001 and $\varepsilon$ to 0.8.

Specifically, we aim to address the following research questions:
\begin{description}
	\item[RQ 1] How does the runtime performance of {\ApproxMCSix} compare to that of {\ApproxMCFour}?
	\item[RQ 2] How does the accuracy of the counts computed by {\ApproxMCSix} compare to that of the exact count?
\end{description}

\paragraph{Summary} In summary, {\ApproxMCSix} consistently outperforms {\ApproxMCFour}. 
Specifically, it solved 204 additional instances and reduced the PAR-2 score by 1063 seconds in comparison to {\ApproxMCFour}. 
The average speedup of {\ApproxMCSix} over {\ApproxMCFour} was 4.68.
In addition, {\ApproxMCSix} provided a high-quality approximation with an average observed error of 0.1, much smaller than the theoretical error tolerance of 0.8.

\subsection{RQ1. Overall Performance}
Figure~\ref{fig: time} compares the counting time of {\ApproxMCSix} and {\ApproxMCFour}.
The $x$-axis represents the index of the instances, sorted in ascending order of runtime,
and the $y$-axis represents the runtime for each instance.
A point $(x,y)$ indicates that  a counter can solve $x$ instances within $y$ seconds.
Thus, for a given time limit $y$, a counter whose curve is on the right has solved more instances than a counter on the left.
It can be seen in the figure that {\ApproxMCSix} consistently outperforms {\ApproxMCFour}.
In total, {\ApproxMCSix} solved 204 more instances than {\ApproxMCFour}.

Table~\ref{tab: overall perf} provides a detailed comparison between {\ApproxMCSix} and {\ApproxMCFour}. 
The first column lists three measures of interest: the number of solved instances, the PAR-2 score, and the speedup of {\ApproxMCSix} over {\ApproxMCFour}. 
The second and third columns show the results for {\ApproxMCFour} and {\ApproxMCSix}, respectively. 
The second column indicates that {\ApproxMCFour} solved 998 of the 1890 instances and achieved a PAR-2 score of 4934. 
The third column shows that {\ApproxMCSix} solved 1202 instances and achieved a PAR-2 score of 3871.
In comparison, {\ApproxMCSix} solved 204 more instances and reduced the PAR-2 score by 1063 seconds in comparison to {\ApproxMCFour}. 
The geometric mean of the speedup for {\ApproxMCSix} over {\ApproxMCFour} is 4.68.
This speedup was calculated only for instances solved by both counters.

\begin{table}[h]
	\centering
	\begin{tabular}{c c c}
		\toprule
		& {\ApproxMCFour} & {\ApproxMCSix} \\
		\midrule
		\#Solved & 998 & 1202 \\
		PAR-2 score & 4934 & 3871 \\
		Speedup & --- & 4.68\\
		\bottomrule
	\end{tabular}
	\caption{The number of solved instances and PAR-2 score for {\ApproxMCSix} versus {\ApproxMCFour} on 1890 instances. The geometric mean of the speedup of {\ApproxMCSix} over {\ApproxMCFour} is also reported.}
	\label{tab: overall perf}
\end{table}

\begin{figure}[h!]
	\centering
	\includegraphics[scale=0.35]{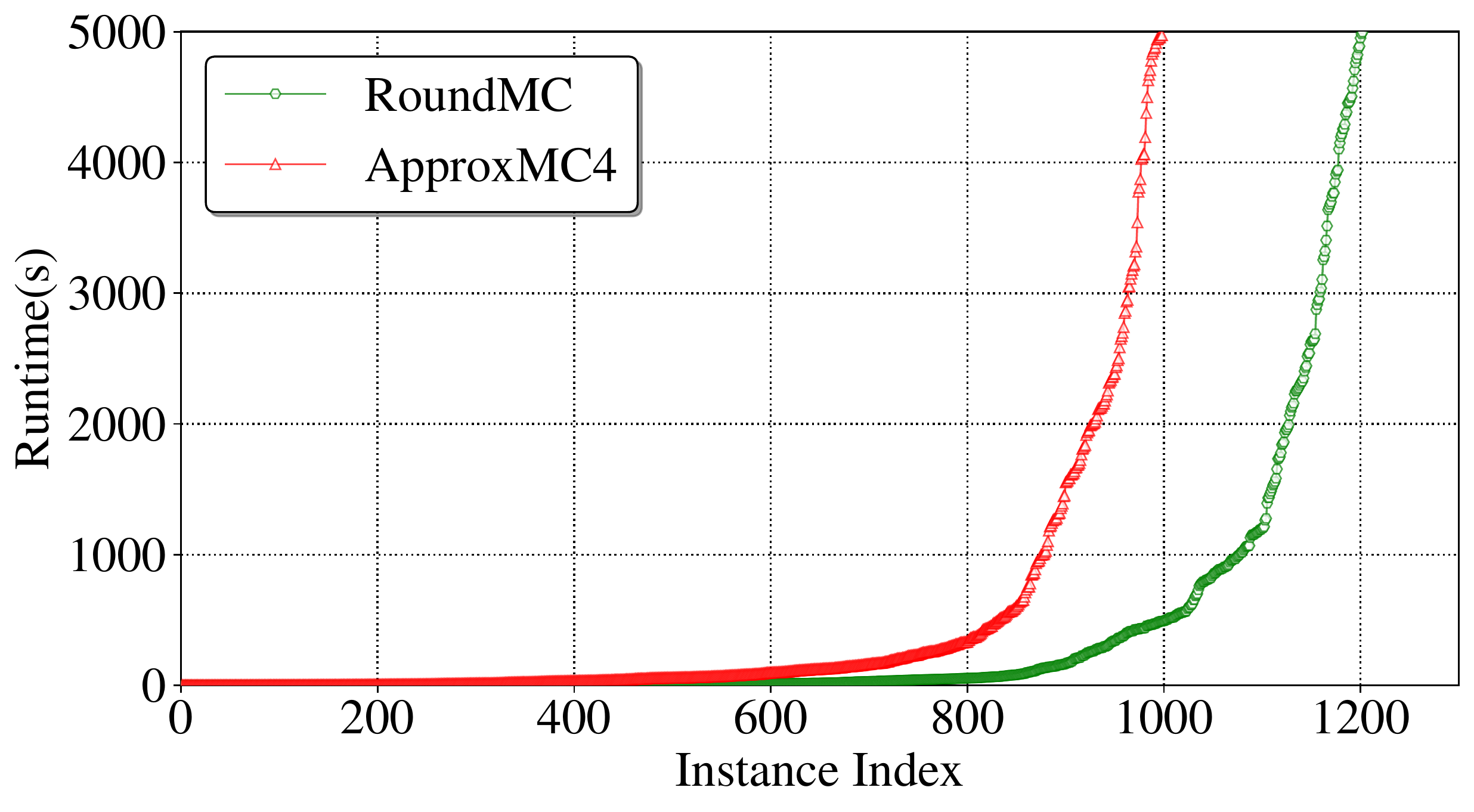}
	\caption{Comparison of counting times for {\ApproxMCSix} and {\ApproxMCFour}.}
	\label{fig: time}
\end{figure}

\subsection{RQ2. Approximation Quality}
We used the state-of-the-art probabilistic exact model counter {\ganak} to compute the exact model count and compare it to the results of {\ApproxMCSix}.
We collected statistics on instances solved by both {\ganak} and {\ApproxMCSix}. 
Figure~\ref{fig: quality} presents results for a subset of instances. 
The x-axis represents the index of instances sorted in ascending order by the number of solutions,
and the y-axis represents the number of solutions in a log scale. 
Theoretically, the approximate count from {\ApproxMCSix} should be within the range of $|\satisfying{F}| \cdot 1.8$ and $|\satisfying{F}|/1.8$ with probability $0.999$, 
where $|\satisfying{F}|$ denotes the exact count returned by {\ganak}.
The range is indicated by the upper and lower bounds, represented by the curves $y=|\satisfying{F}| \cdot 1.8$ and $y=|\satisfying{F}|/1.8$, respectively. 
Figure~\ref{fig: quality} shows that the approximate counts from {\ApproxMCSix} fall within the expected range $\left[|\satisfying{F}|/1.8, |\satisfying{F}|\cdot1.8\right]$ for all instances 
except for four points slightly above the upper bound. 
These four outliers are due to a bug in the preprocessor {\Arjun} that probably depends on the version of the C++ compiler and will be fixed in the future. 
We also calculated the observed error, which is the mean relative difference between the approximate and exact counts in our experiments, i.e., $\max\{\finalestimate/|\satisfying{F}|-1, |\satisfying{F}|/\finalestimate-1\}$. 
The overall observed error was 0.1, which is significantly smaller than the theoretical error tolerance of 0.8.

\begin{figure}[h]
	\centering
	\includegraphics[scale=0.35]{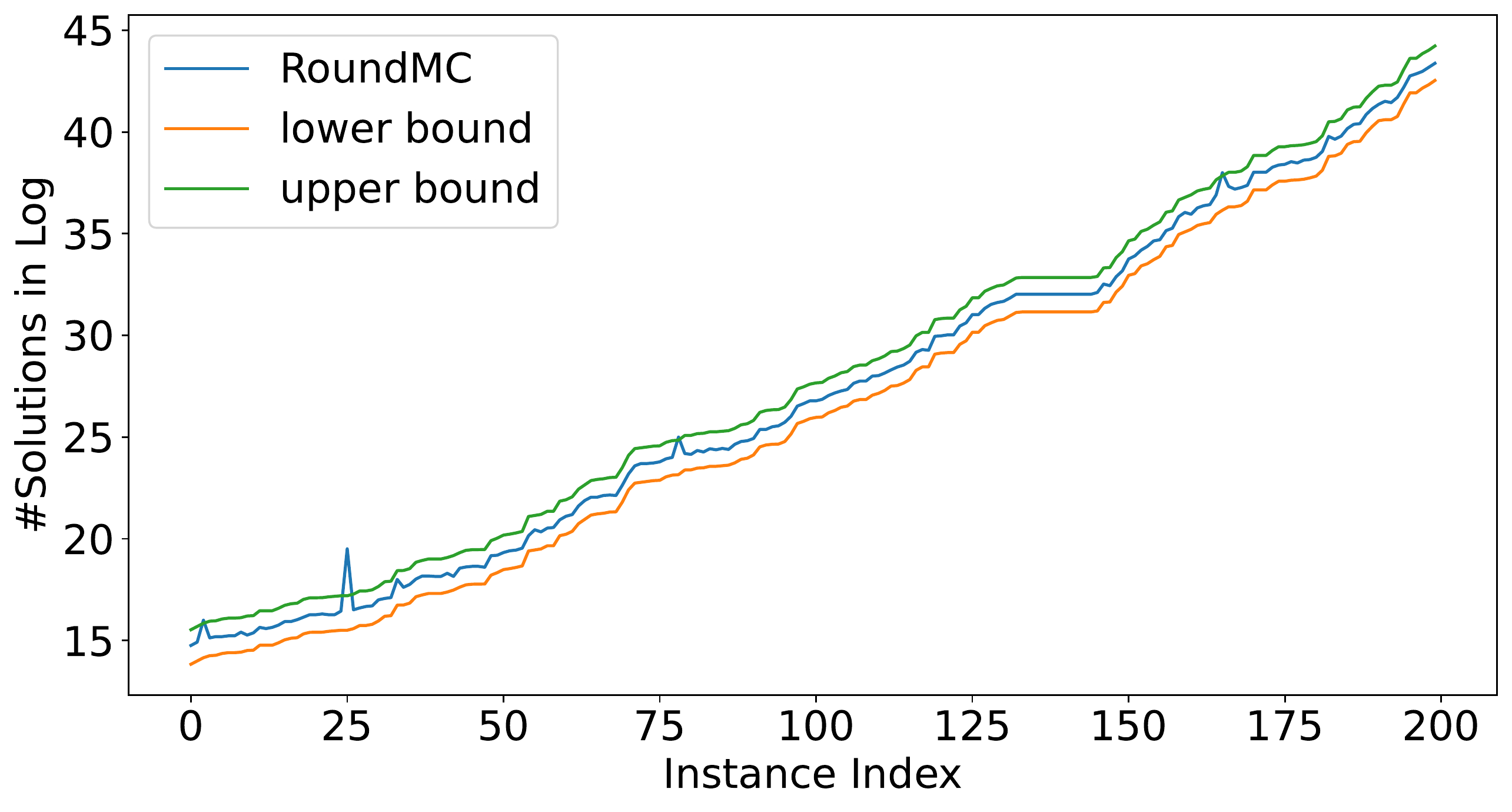}
	\caption{Comparison of approximate counts from {\ApproxMCSix} to exact counts from {\ganak}.}
	\label{fig: quality}
\end{figure}

%% file: sec/conclusion.tex
\section{Conclusion} \label{sec: conclusion}
In this paper, we addressed the scalability challenges faced by {\ApproxMC} in the smaller $\delta$ range. To this end, we proposed a \emph{rounding}-based algorithm, {\ApproxMCSix}, which reduces the number of estimations required by 84\% while providing the same $(\varepsilon,\delta$)-guarantees. Our empirical evaluation on 1890 instances shows that {\ApproxMCSix} solved 204 more instances and achieved a reduction in PAR-2 score of 1063 seconds. Furthermore,  {\ApproxMCSix} achieved a $4\times$ speedup over {\ApproxMC} on the instances that both {\ApproxMCSix} and {\ApproxMC} could solve.

%% file: sec/appendix.tex
\newpage
\appendix

\section{Proof of Proposition~\ref{prop: 2 universal}}
\label{sec: appendix preliminary}
\begin{proof}
	For $\forall y \in \{0,1\}^n, \alpha^{(m)} \in \{0,1\}^m$, let $\gamma_{y,\alpha^{(m)}}$ be an indicator variable that is 1 when $h^{(m)}(y)=\alpha^{(m)}$.
	According to the definition of strongly 2-universal function, we obtain $\forall x,y \in \{0,1\}^n, \Exp{\gamma_{y,\alpha^{(m)}}} = \frac{1}{2^m}$ and $\Exp{\gamma_{x,\alpha^{(m)}}\cdot\gamma_{y,\alpha^{(m)}}} = \frac{1}{2^{2m}}$.
	To prove Equation~\ref{eq: exp}, we obtain 
	\begin{align*}
		\Exp{\Cnt{F}{m}} = \Exp{\sum_{y\in \satisfying{F}} \gamma_{y,\alpha^{(m)}}} = \sum_{y\in \satisfying{F}} \Exp{\gamma_{y,\alpha^{(m)}}} = \frac{|\satisfying{F}|}{2^m}
	\end{align*}
	To prove Equation~\ref{eq: var exp}, we derive 
	\begin{align*}
		\Exp{\Cnt{F}{m}^2} &= \Exp{\sum_{y\in \satisfying{F}} \gamma^2_{y,\alpha^{(m)}} + \sum_{x\not=y\in \satisfying{F}} \gamma_{x,\alpha^{(m)}} \cdot \gamma_{y,\alpha^{(m)}}} \\
						   &= \Exp{\sum_{y\in \satisfying{F}} \gamma_{y,\alpha^{(m)}}} + \sum_{x\not=y\in \satisfying{F}} \Exp{\gamma_{x,\alpha^{(m)}} \cdot \gamma_{y,\alpha^{(m)}}} \\
						   &= \Exp{\Cnt{F}{m}} + \frac{|\satisfying{F}|(|\satisfying{F}|-1)}{2^{2m}}
	\end{align*}
	Then, we obtain 
	\begin{align*}
		\sigma^2\left[\Cnt{F}{m}\right] &= \Exp{\Cnt{F}{m}^2} - \Exp{\Cnt{F}{m}}^2 \\
										&= \Exp{\Cnt{F}{m}} + \frac{|\satisfying{F}|(|\satisfying{F}|-1)}{2^{2m}} - \left(\frac{|\satisfying{F}|}{2^m}\right)^2 \\
										&= \Exp{\Cnt{F}{m}} - \frac{|\satisfying{F}|}{2^{2m}} \\
										&\le \Exp{\Cnt{F}{m}}
	\end{align*}
	\qed
\end{proof}

\section{Weakness of Proposition~\ref{prop:eq-bound}}
The following proposition states that Proposition~\ref{prop:eq-bound} provides a loose upper bound for $\Prb{\Obad_t}$.
\begin{proposition}
	\label{prop: untight}
	Assuming $t$ is odd, we have:
	\begin{align*}
		\Prb{\Obad_t} < \eta(t, \lceil t/2 \rceil, \Prb{L \cup U})
	\end{align*}
\end{proposition}
\begin{proof}
	We will now construct a case counted by $\eta(t, \lceil t/2 \rceil, \Prb{L \cup U})$ but not contained within the event $\Obad_t$.
	Let $I^L_i$ be an indicator variable that is 1 when {\ApproxMCCore} returns a $\solCount$ less than $ \frac{|\satisfying{F}|}{1+\varepsilon}$, indicating the occurrence of event $L$ in the $i$-th repetition.
	Let $I^U_i$ be an indicator variable that is 1 when {\ApproxMCCore} returns a $\solCount$ greater than $(1+\varepsilon)|\satisfying{F}|$, indicating the occurrence of event $U$ in the $i$-th repetition.
	Consider a scenario where $I^L_i=1$ for $i=1, 2, ..., \left\lceil\frac{t}{4}\right\rceil$ 
	, $I^U_j=1$ for $j=\left\lceil\frac{t}{4}\right\rceil+1, ..., \left\lceil\frac{t}{2}\right\rceil$ 
	, and $I^L_k=I^U_k=0$ for $k>\left\lceil\frac{t}{2}\right\rceil$.
	$\eta(t, \lceil t/2 \rceil, \Prb{L \cup U})$ represents $\sum_{i=1}^t(I^L_i \vee I^U_i) \ge \lceil \frac{t}{2} \rceil$.
	We can see that this case is included in $\sum_{i=1}^t(I^L_i \vee I^U_i) \ge \lceil \frac{t}{2} \rceil$ 
	and therefore counted by $\eta(t, \lceil t/2 \rceil, \Prb{L \cup U})$ 
	since there are $\left\lceil\frac{t}{2}\right\rceil$ estimates outside the {\PAC} range.
	However, this case means that $\left\lceil\frac{t}{4}\right\rceil$ estimates fall within the range less than $\frac{|\satisfying{F}|}{1+\varepsilon}$
	and $\left\lceil\frac{t}{2}\right\rceil - \left\lceil\frac{t}{4}\right\rceil$ estimates fall within the range greater than $(1+\varepsilon)|\satisfying{F}|$,
	while the remaining $\left\lfloor \frac{t}{2} \right\rfloor$ estimates correctly fall within the range $\left[ \frac{|\satisfying{F}|}{1+\varepsilon}, (1 + \varepsilon)|\satisfying{F}| \right]$.
	Therefore, after sorting all the estimates, {\ApproxMCSix} returns a correct estimate since the median falls within the {\PAC} range $\left[ \frac{|\satisfying{F}|}{1+\varepsilon}, (1 + \varepsilon)|\satisfying{F}| \right]$.
	In other words, this case is out of the event $\Obad_t$.
	In conclusion, there is a scenario that is out of the event $\Obad_t$, undesirably included in expression $\sum_{i=1}^t(I^L_i \vee I^U_i) \ge \lceil \frac{t}{2} \rceil$ and counted by $\eta(t, \lceil t/2 \rceil, \Prb{L \cup U})$, 
	which means $\Prb{\Obad_t}$ is strictly less than $\eta(t, \lceil t/2 \rceil, \Prb{L \cup U})$.
	\qed
\end{proof}

\section{Proof of $p_{max} \le 0.36$ for {\ApproxMC}}
\label{sec: appendix appmc}
\begin{proof}
	We prove the case of $\sqrt{2}-1\le\varepsilon<1$.
	Similarly to the proof in Section~\ref{sec: proof one case}, we aim to bound $\Prb{L}$ by the following equation:
	\begin{align*}
		\Prb{L} = \left[ \bigcup_{i\in\{1,...,n\}} \left( \overline{T_{i-1}} \cap T_i \cap L_i \right) \right]
		\tag{\ref{eq: L expansion} revisited}
	\end{align*}
	which can be simplified by three observations labeled $O1, O2$ and $O3$ below.
	\begin{description}
		\item[$O1:$] $\forall i \le m^\ast-3, T_i \subseteq T_{i+1}$. Therefore,
		\begin{align*}
			\bigcup_{i\in\{1,...,m^\ast-3\}} (\overline{T_{i-1}} \cap T_i \cap L_i) \subseteq \bigcup_{i\in\{1,...,m^\ast-3\}} T_i \subseteq T_{m^\ast-3}
		\end{align*}
		\item[$O2:$] For $i \in \{m^\ast-2, m^\ast-1\}$, we have
		\begin{align*}
			\bigcup_{i\in\{m^\ast-2, m^\ast-1\}} (\overline{T_{i-1}} \cap T_i \cap L_i) \subseteq L_{m^\ast-2} \cup L_{m^\ast-1}
		\end{align*}
		\item[$O3:$] $\forall i \ge m^\ast$, $\overline{T_i}$ implies $\Cnt{F}{i} > \thresh$ and then we have
		$2^i \times \Cnt{F}{i} > 2^{m^\ast} \times \thresh \ge |\satisfying{F}|\left( 1+\frac{\varepsilon}{1+\varepsilon} \right)$.
		The second inequality follows from $m^\ast \ge \log_2|\satisfying{F}| - \log_2\left(\pivot\right)$.
		Then we obtain $\left( \Cnt{F}{i} > \Exp{\Cnt{F}{i}}\left( 1+\frac{\varepsilon}{1+\varepsilon} \right) \right)$.
		Therefore, $\overline{T_i} \subseteq U'_i$ for $i \ge m^\ast$. 
		Since $\forall i, \overline{T_i} \subseteq \overline{T_{i-1}}$, we have 
		\begin{align*}
			\bigcup_{i\in\{m^\ast,...,n\}} \left( \overline{T_{i-1}} \cap T_i \cap L_i \right) 
			& \subseteq \bigcup_{i\in\{m^\ast+1,...,n\}} \overline{T_{i-1}} \cup ( \overline{T_{m^\ast-1}} \cap T_{m^\ast} \cap L_{m^\ast} ) \\
			& \subseteq \overline{T_{m^\ast}} \cup ( \overline{T_{m^\ast-1}} \cap T_{m^\ast} \cap L_{m^\ast} ) \\
			& \subseteq \overline{T_{m^\ast}} \cup L_{m^\ast} \\
			& \subseteq U'_{m^\ast} \cup L_{m^\ast}
		\end{align*}
	\end{description}
	Following the observations $O1, O2$ and $O3$, we simplify Equation~\ref{eq: L expansion} and obtain
	\begin{align*}
		\Prb{L} \le \Prb{T_{m^\ast-3}} + \Prb{L_{m^\ast-2}} + \Prb{L_{m^\ast-1}} + \Prb{U'_{m^\ast} \cup L_{m^\ast}}
	\end{align*}
	Employing Lemma 2 in~\cite{CMV16} gives $\Prb{L} \le 0.36$. Note that $U$ in~\cite{CMV16} represents $U'$ of our definition.
	
	Then, following the $O4$ and $O5$ in Section~\ref{sec: proof one case}, we obtain
	\begin{align*}
		\Prb{U} \le \Prb{U'_{m^\ast}}
	\end{align*}
	Employing Lemma~\ref{lm: bound} gives $\Prb{U} \le 0.169$. 
	As a result, $p_{max} \le 0.36$.
	\qed
\end{proof}

\section{Proof of Lemma~\ref{thm: complete L U bound}}
\label{sec: appendix roundmc}
We restate the lemma below and prove the statements section by section. The proof for $\sqrt{2}-1\le\varepsilon<1$ has been shown in Section~\ref{sec: proof one case}.
\LUbound*

\subsection{Proof of $\Prb{L} \le 0.262$ for $\varepsilon<\sqrt{2}-1$}
We first consider two cases: $\Exp{\Cnt{F}{m^\ast}} < \frac{1+\varepsilon}{2}\thresh$ and $\Exp{\Cnt{F}{m^\ast}} \ge \frac{1+\varepsilon}{2}\thresh$,
and then merge the results to complete the proof.

\subsubsection{Case 1: $\Exp{\Cnt{F}{m^\ast}} < \frac{1+\varepsilon}{2}\thresh$}
\begin{lemma}
	\label{lm: bound eps<sqrt2-1 case 1}
	Given $\varepsilon<\sqrt{2}-1$, the following bounds hold: 
	\begin{enumerate}
		\item $\Prb{T_{m^\ast-2}} \le \frac{1}{29.67}$
		\item $\Prb{L_{m^\ast-1}} \le \frac{1}{10.84}$
	\end{enumerate}
\end{lemma}

\begin{proof}
	Let's first prove the statement 1.
	For $\varepsilon<\sqrt{2}-1$, we have $\thresh < (2-\frac{\sqrt{2}}{2})\pivot$ and $\Exp{\Cnt{F}{m^\ast-2}} \ge 2\pivot$.
	Therefore, $\Prb{T_{m^\ast-2}} \le \Prb{\Cnt{F}{m^\ast-2} \le (1-\frac{\sqrt{2}}{4})\Exp{\Cnt{F}{m^\ast-2}}}$.
	Finally, employing Lemma~\ref{lm: inequality} with $\beta=1-\frac{\sqrt{2}}{4}$, 
	we obtain $\Prb{T_{m^\ast-2}} \le \frac{1}{1+(\frac{\sqrt{2}}{4})^2\cdot2\pivot} \le \frac{1}{1+(\frac{\sqrt{2}}{4})^2\cdot2\cdot9.84\cdot(1+\frac{1}{\sqrt{2}-1})^2} \le \frac{1}{29.67}$.
	To prove the statement 2, we employ Lemma~\ref{lm: inequality} with $\beta=\frac{1}{1+\varepsilon}$ and $\Exp{\Cnt{F}{m^\ast-1}} \ge \pivot$ to obtain
	$\Prb{L_{m^\ast-1}} \le \frac{1}{1+(1-\frac{1}{1+\varepsilon})^2\cdot\Exp{\Cnt{F}{m^\ast-1}}}
	\le \frac{1}{1+(1-\frac{1}{1+\varepsilon})^2\cdot9.84\cdot(1+\frac{1}{\varepsilon})^2} = \frac{1}{10.84}$.
	\qed
\end{proof}

Then, we prove that $\Prb{L} \le 0.126$ for $\Exp{\Cnt{F}{m^\ast}} < \frac{1+\varepsilon}{2}\thresh$.
\begin{proof}
	We aim to bound $\Prb{L}$ by the following equation:
	\begin{align*}
		\Prb{L} = \left[ \bigcup_{i\in\{1,...,n\}} \left( \overline{T_{i-1}} \cap T_i \cap L_i \right) \right]
		\tag{\ref{eq: L expansion} revisited}
	\end{align*}
	which can be simplified by the three observations labeled $O1, O2$ and $O3$ below.
	\begin{description}
		\item[$O1:$] $\forall i \le m^\ast-2, T_i \subseteq T_{i+1}$. Therefore,
		\begin{align*}
			\bigcup_{i\in\{1,...,m^\ast-2\}} (\overline{T_{i-1}} \cap T_i \cap L_i) \subseteq \bigcup_{i\in\{1,...,m^\ast-2\}} T_i \subseteq T_{m^\ast-2}
		\end{align*}
		\item[$O2:$] For $i = m^\ast-1$, we have
		\begin{align*}
			\overline{T_{m^\ast-2}} \cap T_{m^\ast-1} \cap L_{m^\ast-1} \subseteq L_{m^\ast-1}
		\end{align*}
		\item[$O3:$] $\forall i \ge m^\ast$, since rounding $\Cnt{F}{i}$ up to $\frac{\sqrt{1+2\varepsilon}}{2}\pivot$, we have 
		$\Cnt{F}{i} \ge \frac{\sqrt{1+2\varepsilon}}{2}\pivot \ge \frac{\thresh}{2} > \frac{\Exp{\Cnt{F}{m^\ast}}}{1+\varepsilon} \ge \frac{\Exp{\Cnt{F}{i}}}{1+\varepsilon}$. 
		The second last inequality follows from $\Exp{\Cnt{F}{m^\ast}} < \frac{1+\varepsilon}{2}\thresh$.
		Therefore, $L_i=\emptyset$ for $i \ge m^\ast$ and we have
		\begin{align*}
			\bigcup_{i\in\{m^\ast, ..., n\}} (\overline{T_{i-1}} \cap T_i \cap L_i) = \emptyset
		\end{align*}
	\end{description}
	Following the observations $O1, O2$ and $O3$, we simplify Equation~\ref{eq: L expansion} and obtain
	\begin{align*}
		\Prb{L} \le \Prb{T_{m^\ast-2}} + \Prb{L_{m^\ast-1}}
	\end{align*}
	Employing Lemma~\ref{lm: bound eps<sqrt2-1 case 1} gives $\Prb{L} \le 0.126$.
	\qed
\end{proof}

\subsubsection{Case 2: $\Exp{\Cnt{F}{m^\ast}} \ge \frac{1+\varepsilon}{2}\thresh$}
\begin{lemma}
	\label{lm: bound eps<sqrt2-1 case 2}
	Given $\Exp{\Cnt{F}{m^\ast}} \ge \frac{1+\varepsilon}{2}\thresh$, the following bounds hold: 
	\begin{enumerate}
		\item $\Prb{T_{m^\ast-1}} \le \frac{1}{10.84}$
		\item $\Prb{L_{m^\ast}} \le \frac{1}{5.92}$
	\end{enumerate}
\end{lemma}

\begin{proof}
	Let's first prove the statement 1.
	From $\Exp{\Cnt{F}{m^\ast}} \ge \frac{1+\varepsilon}{2}\thresh$, we can derive $\Exp{\Cnt{F}{m^\ast-1}} \ge (1+\varepsilon)\thresh$.
	Therefore, $\Prb{T_{m^\ast-1}} \le \Prb{\Cnt{F}{m^\ast-1} \le \frac{1}{1+\varepsilon}\Exp{\Cnt{F}{m^\ast-1}}}$.
	Finally, employing Lemma~\ref{lm: inequality} with $\beta=\frac{1}{1+\varepsilon}$, 
	we obtain $\Prb{T_{m^\ast-1}} \le \frac{1}{1+(1-\frac{1}{1+\varepsilon})^2\cdot\Exp{\Cnt{F}{m^\ast-1}}} \le \frac{1}{1+(1-\frac{1}{1+\varepsilon})^2\cdot(1+\varepsilon)\thresh} 
	= \frac{1}{1+9.84(1+2\varepsilon)} \le \frac{1}{10.84}$.
	To prove the statement 2, we employ Lemma~\ref{lm: inequality} with $\beta=\frac{1}{1+\varepsilon}$ and $\Exp{\Cnt{F}{m^\ast}} \ge \frac{1+\varepsilon}{2}\thresh$ to obtain
	$\Prb{L_{m^\ast}} \le \frac{1}{1+(1-\frac{1}{1+\varepsilon})^2\cdot\Exp{\Cnt{F}{m^\ast}}}
	\le \frac{1}{1+(1-\frac{1}{1+\varepsilon})^2\cdot\frac{1+\varepsilon}{2}\thresh} 
	= \frac{1}{1+4.92(1+2\varepsilon)} \le \frac{1}{5.92}$.
	\qed
\end{proof}

Then, we prove that $\Prb{L} \le 0.262$ for $\Exp{\Cnt{F}{m^\ast}} \ge \frac{1+\varepsilon}{2}\thresh$.
\begin{proof}
	We aim to bound $\Prb{L}$ by the following equation:
	\begin{align*}
		\Prb{L} = \left[ \bigcup_{i\in\{1,...,n\}} \left( \overline{T_{i-1}} \cap T_i \cap L_i \right) \right]
		\tag{\ref{eq: L expansion} revisited}
	\end{align*}
	which can be simplified by the three observations labeled $O1, O2$ and $O3$ below.
	\begin{description}
		\item[$O1:$] $\forall i \le m^\ast-1, T_i \subseteq T_{i+1}$. Therefore,
		\begin{align*}
			\bigcup_{i\in\{1,...,m^\ast-1\}} (\overline{T_{i-1}} \cap T_i \cap L_i) \subseteq \bigcup_{i\in\{1,...,m^\ast-1\}} T_i \subseteq T_{m^\ast-1}
		\end{align*}
		\item[$O2:$] For $i = m^\ast$, we have
		\begin{align*}
			\overline{T_{m^\ast-1}} \cap T_{m^\ast} \cap L_{m^\ast} \subseteq L_{m^\ast}
		\end{align*}
		\item[$O3:$] $\forall i \ge m^\ast+1$, since rounding $\Cnt{F}{i}$ up to $\frac{\sqrt{1+2\varepsilon}}{2}\pivot$ and $m^\ast \ge \log_2|\satisfying{F}| - \log_2\left(\pivot\right)$, we have 
		$2^i \times \Cnt{F}{i} \ge 2^{m^\ast+1} \times \frac{\sqrt{1+2\varepsilon}}{2}\pivot \ge \sqrt{1+2\varepsilon}|\satisfying{F}| \ge \frac{|\satisfying{F}|}{1+\varepsilon}$.
		Then we have $\left( \Cnt{F}{i} \ge \frac{\Exp{\Cnt{F}{i}}}{1+\varepsilon} \right)$.
		Therefore, $L_i=\emptyset$ for $i \ge m^\ast+1$ and we have
		\begin{align*}
			\bigcup_{i\in\{m^\ast+1, ..., n\}} (\overline{T_{i-1}} \cap T_i \cap L_i) = \emptyset
		\end{align*}
	\end{description}
	Following the observations $O1, O2$ and $O3$, we simplify Equation~\ref{eq: L expansion} and obtain
	\begin{align*}
		\Prb{L} \le \Prb{T_{m^\ast-1}} + \Prb{L_{m^\ast}}
	\end{align*}
	Employing Lemma~\ref{lm: bound eps<sqrt2-1 case 2} gives $\Prb{L} \le 0.262$.
	\qed
\end{proof}

Combining the Case 1 and 2, we obtain $\Prb{L} \le \max\{0.126, 0.262\}=0.262$. Therefore, we prove the statement for {\ApproxMCSix}: $\Prb{L} \le 0.262$ for $\varepsilon<\sqrt{2}-1$.

\subsection{Proof of $\Prb{L} \le 0.085$ for $1\le\varepsilon<3$}
\begin{lemma}
	\label{lm: bound 1<=eps<3}
	Given $1\le\varepsilon<3$, the following bounds hold: 
	\begin{enumerate}
		\item $\Prb{T_{m^\ast-4}} \le \frac{1}{86.41}$
		\item $\Prb{L_{m^\ast-3}} \le \frac{1}{40.36}$
		\item $\Prb{L_{m^\ast-2}} \le \frac{1}{20.68}$
	\end{enumerate}
\end{lemma}

\begin{proof}
	Let's first prove the statement 1.
	For $\varepsilon < 3$, we have $\thresh < \frac{7}{4} \pivot$ and $\Exp{\Cnt{F}{m^\ast-4}} \ge 8\pivot$.
	Therefore, $\Prb{T_{m^\ast-4}} \le \Prb{\Cnt{F}{m^\ast-4} \le \frac{7}{32}\Exp{\Cnt{F}{m^\ast-4}}}$.
	Finally, employing Lemma~\ref{lm: inequality} with $\beta=\frac{7}{32}$, 
	we obtain $\Prb{T_{m^\ast-4}} \le \frac{1}{1+(1-\frac{7}{32})^2\cdot8\pivot} \le \frac{1}{1+(1-\frac{7}{32})^2\cdot8\cdot9.84\cdot(1+\frac{1}{3})^2} \le \frac{1}{86.41}$.
	To prove the statement 2, we employ Lemma~\ref{lm: inequality} with $\beta=\frac{1}{1+\varepsilon}$ and $\Exp{\Cnt{F}{m^\ast-3}} \ge 4\pivot$ to obtain
	$\Prb{L_{m^\ast-3}} \le \frac{1}{1+(1-\frac{1}{1+\varepsilon})^2\cdot\Exp{\Cnt{F}{m^\ast-3}}}
	\le \frac{1}{1+(1-\frac{1}{1+\varepsilon})^2\cdot4\cdot9.84\cdot(1+\frac{1}{\varepsilon})^2} = \frac{1}{40.36}$.
	Following the proof of Lemma 2 in~\cite{CMV16} we can prove the statement 3. 
	\qed
\end{proof}

Now let us prove the statement for {\ApproxMCSix}: $\Prb{L} \le 0.085$ for $1\le\varepsilon<3$ .
\begin{proof}
	We aim to bound $\Prb{L}$ by the following equation:
	\begin{align*}
	\Prb{L} = \left[ \bigcup_{i\in\{1,...,n\}} \left( \overline{T_{i-1}} \cap T_i \cap L_i \right) \right]
	\tag{\ref{eq: L expansion} revisited}
	\end{align*}
	which can be simplified by the three observations labeled $O1, O2$ and $O3$ below.
	\begin{description}
		\item[$O1:$] $\forall i \le m^\ast-4, T_i \subseteq T_{i+1}$. Therefore,
		\begin{align*}
			\bigcup_{i\in\{1,...,m^\ast-4\}} (\overline{T_{i-1}} \cap T_i \cap L_i) \subseteq \bigcup_{i\in\{1,...,m^\ast-4\}} T_i \subseteq T_{m^\ast-4}
		\end{align*}
		\item[$O2:$] For $i \in \{m^\ast-3, m^\ast-2\}$, we have
		\begin{align*}
			\bigcup_{i\in\{m^\ast-3, m^\ast-2\}} (\overline{T_{i-1}} \cap T_i \cap L_i) \subseteq L_{m^\ast-3} \cup L_{m^\ast-2}
		\end{align*}
		\item[$O3:$] $\forall i \ge m^\ast-1$, since rounding $\Cnt{F}{i}$ up to $\pivot$ and $m^\ast \ge \log_2|\satisfying{F}| - \log_2\left(\pivot\right)$, we have 
		$2^i \times \Cnt{F}{i} \ge 2^{m^\ast-1} \times \pivot \ge \frac{|\satisfying{F}|}{2} \ge \frac{|\satisfying{F}|}{1+\varepsilon}$. The last inequality follows from $\varepsilon \ge 1$.
		Then we have $\left( \Cnt{F}{i} \ge \frac{\Exp{\Cnt{F}{i}}}{1+\varepsilon} \right)$.
		Therefore, $L_i=\emptyset$ for $i \ge m^\ast-1$ and we have
		\begin{align*}
			\bigcup_{i\in\{m^\ast-1, ..., n\}} (\overline{T_{i-1}} \cap T_i \cap L_i) = \emptyset
		\end{align*}
	\end{description}
	Following the observations $O1, O2$ and $O3$, we simplify Equation~\ref{eq: L expansion} and obtain
	\begin{align*}
		\Prb{L} \le \Prb{T_{m^\ast-4}} + \Prb{L_{m^\ast-3}} + \Prb{L_{m^\ast-2}}
	\end{align*}
	Employing Lemma~\ref{lm: bound 1<=eps<3} gives $\Prb{L} \le 0.085$.
	\qed
\end{proof}

\subsection{Proof of $\Prb{L} \le 0.055$ for $3\le\varepsilon<4\sqrt{2}-1$}
\begin{lemma}
	\label{lm: bound 3<=eps<4sqrt2-1}
	Given $3\le\varepsilon<4\sqrt{2}-1$, the following bound hold: 
	\begin{align*}
		\Prb{T_{m^\ast-3}} \le \frac{1}{18.19}
	\end{align*}
\end{lemma}

\begin{proof}
	For $\varepsilon < 4\sqrt{2}-1$, we have $\thresh < (2-\frac{\sqrt{2}}{8}) \pivot$ and $\Exp{\Cnt{F}{m^\ast-3}} \ge 4\pivot$.
	Therefore, $\Prb{T_{m^\ast-3}} \le \Prb{\Cnt{F}{m^\ast-3} \le (\frac{1}{2}-\frac{\sqrt{2}}{32})\Exp{\Cnt{F}{m^\ast-3}}}$.
	Finally, employing Lemma~\ref{lm: inequality} with $\beta=\frac{1}{2}-\frac{\sqrt{2}}{32}$, 
	we obtain $\Prb{T_{m^\ast-3}} \le \frac{1}{1+(1-(\frac{1}{2}-\frac{\sqrt{2}}{32}))^2\cdot4\pivot} \le \frac{1}{1+(1-(\frac{1}{2}-\frac{\sqrt{2}}{32}))^2\cdot4\cdot9.84\cdot(1+\frac{1}{4\sqrt{2}-1})^2} \le \frac{1}{18.19}$.
	\qed
\end{proof}

Now let us prove the statement for {\ApproxMCSix}: $\Prb{L} \le 0.055$ for $3\le\varepsilon<4\sqrt{2}-1$.
\begin{proof}
	We aim to bound $\Prb{L}$ by the following equation:
	\begin{align*}
		\Prb{L} = \left[ \bigcup_{i\in\{1,...,n\}} \left( \overline{T_{i-1}} \cap T_i \cap L_i \right) \right]
		\tag{\ref{eq: L expansion} revisited}
	\end{align*}
	which can be simplified by the two observations labeled $O1$ and $O2$ below.
	\begin{description}
		\item[$O1:$] $\forall i \le m^\ast-3, T_i \subseteq T_{i+1}$. Therefore,
		\begin{align*}
			\bigcup_{i\in\{1,...,m^\ast-3\}} (\overline{T_{i-1}} \cap T_i \cap L_i) \subseteq \bigcup_{i\in\{1,...,m^\ast-3\}} T_i \subseteq T_{m^\ast-3}
		\end{align*}
		\item[$O2:$] $\forall i \ge m^\ast-2$, since rounding $\Cnt{F}{i}$ to $\pivot$ and $m^\ast \ge \log_2|\satisfying{F}| - \log_2\left(\pivot\right)$, we have 
		$2^i \times \Cnt{F}{i} \ge 2^{m^\ast-2} \times \pivot \ge \frac{|\satisfying{F}|}{4} \ge \frac{|\satisfying{F}|}{1+\varepsilon}$. The last inequality follows from $\varepsilon \ge 3$.
		Then we have $\left( \Cnt{F}{i} \ge \frac{\Exp{\Cnt{F}{i}}}{1+\varepsilon} \right)$.
		Therefore, $L_i=\emptyset$ for $i \ge m^\ast-2$ and we have
		\begin{align*}
			\bigcup_{i\in\{m^\ast-2, ..., n\}} (\overline{T_{i-1}} \cap T_i \cap L_i) = \emptyset
		\end{align*}
	\end{description}
	Following the observations $O1$ and $O2$, we simplify Equation~\ref{eq: L expansion} and obtain
	\begin{align*}
		\Prb{L} \le \Prb{T_{m^\ast-3}}
	\end{align*}
	Employing Lemma~\ref{lm: bound 3<=eps<4sqrt2-1} gives $\Prb{L} \le 0.055$.
	\qed
\end{proof}

\subsection{Proof of $\Prb{L} \le 0.023$ for $\varepsilon\ge4\sqrt{2}-1$}
\begin{lemma}
	\label{lm: bound eps>=4sqrt2-1}
	Given $\varepsilon\ge4\sqrt{2}-1$, the following bound hold: 
	\begin{align*}
		\Prb{T_{m^\ast-4}} \le \frac{1}{45.28}
	\end{align*}
\end{lemma}

\begin{proof}
	We have $\thresh < 2\pivot$ and $\Exp{\Cnt{F}{m^\ast-4}} \ge 8\pivot$.
	Therefore, $\Prb{T_{m^\ast-4}} \le \Prb{\Cnt{F}{m^\ast-4} \le \frac{1}{4}\Exp{\Cnt{F}{m^\ast-4}}}$.
	Finally, employing Lemma~\ref{lm: inequality} with $\beta=\frac{1}{4}$, 
	we obtain $\Prb{T_{m^\ast-4}} \le \frac{1}{1+(1-\frac{1}{4})^2\cdot8\pivot} \le \frac{1}{1+(1-\frac{1}{4})^2\cdot8\cdot9.84} \le \frac{1}{45.28}$.
	\qed
\end{proof}

Now let us prove the statement for {\ApproxMCSix}: $\Prb{L} \le 0.023$ for $\varepsilon\ge4\sqrt{2}-1$.
\begin{proof}
	We aim to bound $\Prb{L}$ by the following equation:
	\begin{align*}
		\Prb{L} = \left[ \bigcup_{i\in\{1,...,n\}} \left( \overline{T_{i-1}} \cap T_i \cap L_i \right) \right]
		\tag{\ref{eq: L expansion} revisited}
	\end{align*}
	which can be simplified by the two observations labeled $O1$ and $O2$ below.
	\begin{description}
		\item[$O1:$] $\forall i \le m^\ast-4, T_i \subseteq T_{i+1}$. Therefore,
		\begin{align*}
			\bigcup_{i\in\{1,...,m^\ast-4\}} (\overline{T_{i-1}} \cap T_i \cap L_i) \subseteq \bigcup_{i\in\{1,...,m^\ast-4\}} T_i \subseteq T_{m^\ast-4}
		\end{align*}
		\item[$O2:$] $\forall i \ge m^\ast-3$, since rounding $\Cnt{F}{i}$ to $\sqrt{2}\pivot$ and $m^\ast \ge \log_2|\satisfying{F}| - \log_2\left(\pivot\right)$, we have 
		$2^i \times \Cnt{F}{i} \ge 2^{m^\ast-3} \times \sqrt{2}\pivot \ge \frac{\sqrt{2}|\satisfying{F}|}{8} \ge \frac{|\satisfying{F}|}{1+\varepsilon}$. The last inequality follows from $\varepsilon \ge 4\sqrt{2}-1$.
		Then we have $\left( \Cnt{F}{i} \ge \frac{\Exp{\Cnt{F}{i}}}{1+\varepsilon} \right)$.
		Therefore, $L_i=\emptyset$ for $i \ge m^\ast-3$ and we have
		\begin{align*}
			\bigcup_{i\in\{m^\ast-3, ..., n\}} (\overline{T_{i-1}} \cap T_i \cap L_i) = \emptyset
		\end{align*}
	\end{description}
	Following the observations $O1$ and $O2$, we simplify Equation~\ref{eq: L expansion} and obtain
	\begin{align*}
		\Prb{L} \le \Prb{T_{m^\ast-4}}
	\end{align*}
	Employing Lemma~\ref{lm: bound eps>=4sqrt2-1} gives $\Prb{L} \le 0.023$.
	\qed
\end{proof}

\subsection{Proof of $\Prb{U} \le 0.169$ for $\varepsilon<3$}
\begin{lemma}
	\label{lm: bound eps<3}
	\begin{align*}
		\Prb{U'_{m^\ast}} \le \frac{1}{5.92}
	\end{align*}
\end{lemma}

\begin{proof}
	Employing Lemma~\ref{lm: inequality} with $\gamma = (1+\frac{\varepsilon}{1+\varepsilon})$
	and $\Exp{\Cnt{F}{m^\ast}} \ge \pivot/2$, 
	we obtain $\Prb{U'_{m^\ast}} \le \frac{1}{1+\left(\frac{\varepsilon}{1+\varepsilon}\right)^2\pivot/2} 
	\le \frac{1}{1+9.84/2} \le \frac{1}{5.92}$.
	\qed
\end{proof}

Now let us prove the statement for {\ApproxMCSix}: $\Prb{U} \le 0.169$ for $\varepsilon<3$.
\begin{proof}
	We aim to bound $\Prb{U}$ by the following equation:
	\begin{align*}
		\Prb{U} = \left[ \bigcup_{i\in\{1,...,n\}} \left( \overline{T_{i-1}} \cap T_i \cap U_i \right) \right]
		\tag{\ref{eq: U expansion} revisited}
	\end{align*}
	We derive the following observations $O1$ and $O2$.
	\begin{description}
		\item[$O1:$] $\forall i \le m^\ast-1$, since $m^\ast \le \log_2|\satisfying{F}| - \log_2\left(\pivot\right) + 1$, 
		we have $2^i \times \Cnt{F}{i} \le 2^{m^\ast-1} \times \thresh \le |\satisfying{F}|\left( 1+\frac{\varepsilon}{1+\varepsilon} \right)$.
		Then we obtain $\left( \Cnt{F}{i} \le \Exp{\Cnt{F}{i}}\left( 1+\frac{\varepsilon}{1+\varepsilon} \right) \right)$.
		Therefore, $T_i \cap U'_i = \emptyset$ for $i \le m^\ast-1$ and we have
		\begin{align*}
			\bigcup_{i\in\{1,...,m^\ast-1\}} \left( \overline{T_{i-1}} \cap T_i \cap U_i \right)
			\subseteq \bigcup_{i\in\{1,...,m^\ast-1\}} \left( \overline{T_{i-1}} \cap T_i \cap U'_i \right) = \emptyset
		\end{align*}
		\item[$O2:$] $\forall i \ge m^\ast$, $\overline{T_i}$ implies $\Cnt{F}{i} > \thresh$ and then we have
		$2^i \times \Cnt{F}{i} > 2^{m^\ast} \times \thresh \ge |\satisfying{F}|\left( 1+\frac{\varepsilon}{1+\varepsilon} \right)$.
		The second inequality follows from $m^\ast \ge \log_2|\satisfying{F}| - \log_2\left(\pivot\right)$.
		Then we obtain $\left( \Cnt{F}{i} > \Exp{\Cnt{F}{i}}\left( 1+\frac{\varepsilon}{1+\varepsilon} \right) \right)$.
		Therefore, $\overline{T_i} \subseteq U'_i$ for $i \ge m^\ast$. 
		Since $\forall i, \overline{T_i} \subseteq \overline{T_{i-1}}$, we have 
		\begin{align}
			\bigcup_{i\in\{m^\ast,...,n\}} \left( \overline{T_{i-1}} \cap T_i \cap U_i \right) 
			& \subseteq \bigcup_{i\in\{m^\ast+1,...,n\}} \overline{T_{i-1}} \cup ( \overline{T_{m^\ast-1}} \cap T_{m^\ast} \cap U_{m^\ast} ) \nonumber \\
			& \subseteq \overline{T_{m^\ast}} \cup ( \overline{T_{m^\ast-1}} \cap T_{m^\ast} \cap U_{m^\ast} ) \nonumber \\
			& \subseteq \overline{T_{m^\ast}} \cup U_{m^\ast} \nonumber \\
			& \subseteq U'_{m^\ast} \label{eq: round U eps<3}
		\end{align}
		Remark that for $\varepsilon < \sqrt{2}-1$, we round $\Cnt{F}{m^\ast}$ up to $\frac{\sqrt{1+2\varepsilon}}{2}\pivot$ 
		and we have $2^{m^\ast}\times\frac{\sqrt{1+2\varepsilon}}{2}\pivot \le |\satisfying{F}|(1+\varepsilon)$.
		For $\sqrt{2}-1\le\varepsilon<1$, we round $\Cnt{F}{m^\ast}$ up to $\frac{\pivot}{\sqrt{2}}$ 
		and we have $2^{m^\ast}\times\frac{\pivot}{\sqrt{2}} \le |\satisfying{F}|(1+\varepsilon)$.
		For $1\le\varepsilon<3$, we round $\Cnt{F}{m^\ast}$ up to $\pivot$ 
		and we have $2^{m^\ast}\times\pivot \le |\satisfying{F}|(1+\varepsilon)$.
		The analysis means \emph{rounding} doesn't affect the event $U_{m^\ast}$ and therefore Inequality~\ref{eq: round U eps<3} still holds.
	\end{description}
	Following the observations $O1$ and $O2$, we simplify Equation~\ref{eq: U expansion} and obtain
	\begin{align*}
		\Prb{U} \le \Prb{U'_{m^\ast}}
	\end{align*}
	Employing Lemma~\ref{lm: bound eps<3} gives $\Prb{U} \le 0.169$.
	\qed
\end{proof}

\subsection{Proof of $\Prb{U} \le 0.044$ for $\varepsilon\ge3$}
\begin{lemma}
	\label{lm: bound eps>=3}
	\begin{align*}
		\Prb{\overline{T_{m^\ast+1}}} \le \frac{1}{23.14}
	\end{align*}
\end{lemma}

\begin{proof}
	Since $\Exp{\Cnt{F}{m^\ast+1}} \le \frac{\pivot}{2}$,
	we have $\Prb{\overline{T_{m^\ast+1}}} \le \Prb{\Cnt{F}{m^\ast+1} > 2(1+\frac{\varepsilon}{1+\varepsilon})\Exp{\Cnt{F}{m^\ast+1}}}$.
	Employing Lemma~\ref{lm: inequality} with $\gamma = 2(1+\frac{\varepsilon}{1+\varepsilon})$ and $\Exp{\Cnt{F}{m^\ast+1}} \ge \frac{\pivot}{4}$,
	we obtain $\Prb{\overline{T_{m^\ast+1}}} \le \frac{1}{1+\left(1+\frac{2\varepsilon}{1+\varepsilon}\right)^2\pivot/4}
	= \frac{1}{1+2.46 \cdot \left(3+\frac{1}{\varepsilon}\right)^2} \le \frac{1}{1+2.46 \cdot 3^2} \le \frac{1}{23.14}$.
	\qed
\end{proof}

Now let us prove the statement for {\ApproxMCSix}: $\Prb{U} \le 0.044$ for $\varepsilon\ge3$.
\begin{proof}
	We aim to bound $\Prb{U}$ by the following equation:
	\begin{align*}
		\Prb{U} = \left[ \bigcup_{i\in\{1,...,n\}} \left( \overline{T_{i-1}} \cap T_i \cap U_i \right) \right]
		\tag{\ref{eq: U expansion} revisited}
	\end{align*}
	We derive the following observations $O1$ and $O2$.
	\begin{description}
		\item[$O1:$] $\forall i \le m^\ast+1$, for $3\le\varepsilon<4\sqrt{2}-1$, because we round $\Cnt{F}{i}$ to $\pivot$ and have $m^\ast \le \log_2|\satisfying{F}| - \log_2\left(\pivot\right) + 1$, 
		we obtain $2^i \times \Cnt{F}{i} \le 2^{m^\ast+1} \times \pivot \le 4\cdot|\satisfying{F}| \le (1+\varepsilon)|\satisfying{F}|$.
		For $\varepsilon\ge4\sqrt{2}-1$, we round $\Cnt{F}{i}$ to $\sqrt{2}\pivot$ 
		and obtain $2^i \times \Cnt{F}{i} \le 2^{m^\ast+1} \times \sqrt{2}\pivot \le 4\sqrt{2}\cdot|\satisfying{F}| \le (1+\varepsilon)|\satisfying{F}|$.
		Then, we obtain $\Cnt{F}{i} \le \Exp{\Cnt{F}{i}}(1+\varepsilon)$.
		Therefore, $U_i = \emptyset$ for $i \le m^\ast+1$ and we have
		\begin{align*}
			\bigcup_{i\in\{1,...,m^\ast+1\}} \left( \overline{T_{i-1}} \cap T_i \cap U_i \right) = \emptyset
		\end{align*}
		\item[$O2:$] $\forall i \ge m^\ast+2$, 
		since $\forall i, \overline{T_i} \subseteq \overline{T_{i-1}}$, we have 
		\begin{align*}
			\bigcup_{i\in\{m^\ast+2,...,n\}} \left( \overline{T_{i-1}} \cap T_i \cap U_i \right) 
			\subseteq \bigcup_{i\in\{m^\ast+2,...,n\}} \overline{T_{i-1}} \subseteq \overline{T_{m^\ast+1}}
		\end{align*}
	\end{description}
	Following the observations $O1$ and $O2$, we simplify Equation~\ref{eq: U expansion} and obtain
	\begin{align*}
		\Prb{U} \le \Prb{\overline{T_{m^\ast+1}}}
	\end{align*}
	Employing Lemma~\ref{lm: bound eps>=3} gives $\Prb{U} \le 0.044$.
	\qed
\end{proof}

%% file: main.bbl
\begin{thebibliography}{10}
\providecommand{\url}[1]{\texttt{#1}}
\providecommand{\urlprefix}{URL }
\providecommand{\doi}[1]{https://doi.org/#1}

\bibitem{ABJM+13}
Alur, R., Bodik, R., Juniwal, G., Martin, M.M.K., Raghothaman, M., Seshia,
  S.A., Singh, R., Solar-Lezama, A., Torlak, E., Udupa, A.: Syntax-guided
  synthesis. In: Proc. of FMCAD (2013)

\bibitem{BSSM+19}
Baluta, T., Shen, S., Shine, S., Meel, K.S., Saxena, P.: Quantitative
  verification of neural networks and its security applications. In: Proc. of
  CCS (2019)

\bibitem{BZG20}
Beck, G., Zinkus, M., Green, M.: Automating the development of chosen
  ciphertext attacks. In: Proc. of USENIX Security (2020)

\bibitem{CW77}
Carter, J.L., Wegman, M.N.: Universal classes of hash functions (1977)

\bibitem{CFMS+14}
Chakraborty, S., Fremont, D.J., Meel, K.S., Seshia, S.A., Vardi, M.Y.:
  Distribution-aware sampling and weighted model counting for {SAT}. In: Proc.
  of AAAI (2014)

\bibitem{CMMV16}
Chakraborty, S., Meel, K.S., Mistry, R., Vardi, M.Y.: Approximate probabilistic
  inference via word-level counting. In: Proc. of AAAI (2016)

\bibitem{CMV13}
Chakraborty, S., Meel, K.S., Vardi, M.Y.: A scalable approximate model counter.
  In: Proc. of CP (2013)

\bibitem{CMV16}
Chakraborty, S., Meel, K.S., Vardi, M.Y.: Algorithmic improvements in
  approximate counting for probabilistic inference: From linear to logarithmic
  {SAT} calls. In: Proc. of IJCAI (2016)

\bibitem{DMPV17}
Duenas-Osorio, L., Meel, K.S., Paredes, R., Vardi, M.Y.: Counting-based
  reliability estimation for power-transmission grids. In: Proc. of AAAI (2017)

\bibitem{EGSS13b}
Ermon, S., Gomes, C.P., Sabharwal, A., Selman, B.: Embed and project: Discrete
  sampling with universal hashing. In: Proc. of NeurIPS (2013)

\bibitem{EGSS13a}
Ermon, S., Gomes, C.P., Sabharwal, A., Selman, B.: Taming the curse of
  dimensionality: Discrete integration by hashing and optimization. In: Proc.
  of ICML (2013)

\bibitem{FHH20}
Fichte, J.K., Hecher, M., Hamiti, F.: The model counting competition 2020. ACM
  J. Exp. Algorithmics  (2021)

\bibitem{GVF22}
Gittis, A., Vin, E., Fremont, D.J.: Randomized synthesis for diversity and cost
  constraints with control improvisation. In: Proc. of CAV (2022)

\bibitem{GSS06}
Gomes, C.P., Sabharwal, A., Selman, B.: Model counting: A new strategy for
  obtaining good bounds. In: Proc. of AAAI (2006)

\bibitem{HF21}
Hecher, M., Fichte, J.K.: Model counting competition 2021 (2021),
  \url{https://mccompetition.org/2021/mc_description}

\bibitem{HF22}
Hecher, M., Fichte, J.K.: Model counting competition 2022 (2022),
  \url{https://mccompetition.org/2022/mc_description}

\bibitem{IMMV16}
Ivrii, A., Malik, S., Meel, K.S., Vardi, M.Y.: On computing minimal independent
  support and its applications to sampling and counting. Constraints  (2016)

\bibitem{AM20}
Meel, K.S., Akshay, S.: Sparse hashing for scalable approximate model counting:
  Theory and practice. In: Proc. of LICS (2020)

\bibitem{MVCF+15}
Meel, K.S., Vardi, M.Y., Chakraborty, S., Fremont, D.J., Seshia, S.A., Fried,
  D., Ivrii, A., Malik, S.: Constrained sampling and counting: Universal
  hashing meets sat solving. In: Proc. of Workshop on Beyond NP(BNP) (2016)

\bibitem{R96}
Roth, D.: On the hardness of approximate reasoning. Artificial Intelligence
  (1996)

\bibitem{SBK05}
Sang, T., Bearne, P., Kautz, H.: Performing bayesian inference by weighted
  model counting. In: Proc. of AAAI (2005)

\bibitem{SGM20}
Soos, M., Gocht, S., Meel, K.S.: Tinted, detached, and lazy cnf-xor solving and
  its applications to counting and sampling. In: Proc. of CAV (2020)

\bibitem{SM19}
Soos, M., Meel, K.S.: Bird: Engineering an efficient cnf-xor sat solver and its
  applications to approximate model counting. In: Proc. of AAAI (2019)

\bibitem{SM22}
Soos, M., Meel, K.S.: Arjun: An efficient independent support computation
  technique and its applications to counting and sampling. In: Proc. of ICCAD
  (2022)

\bibitem{S83}
Stockmeyer, L.: The complexity of approximate counting. In: Proc. of STOC
  (1983)

\bibitem{TW21}
Teuber, S., Weigl, A.: Quantifying software reliability via model-counting. In:
  Proc. of QEST (2021)

\bibitem{T89}
Toda, S.: On the computational power of pp and (+)p. In: Proc. of FOCS (1989)

\bibitem{V79}
Valiant, L.G.: The complexity of enumeration and reliability problems. SIAM
  Journal on Computing  (1979)

\bibitem{YCM22}
Yang, J., Chakraborty, S., Meel, K.S.: Projected model counting: Beyond
  independent support. In: Proc. of ATVA (2022)

\bibitem{YM21}
Yang, J., Meel, K.S.: Engineering an efficient pb-xor solver. In: Proc. of CP
  (2021)

\end{thebibliography}
